\documentclass[sigconf]{acmart}

  \providecommand\BibTeX{{%
    \normalfont B\kern-0.5em{\scshape i\kern-0.25em b}\kern-0.8em\TeX}}

\copyrightyear{2024}
\acmYear{2024}
\setcopyright{acmlicensed}\acmConference[KDD '24]{Proceedings of the 30th ACM SIGKDD Conference on Knowledge Discovery and Data Mining}{August 25--29, 2024}{Barcelona, Spain}
\acmBooktitle{Proceedings of the 30th ACM SIGKDD Conference on Knowledge Discovery and Data Mining (KDD '24), August 25--29, 2024, Barcelona, Spain}
\acmDOI{10.1145/3637528.3671526}
\acmISBN{979-8-4007-0490-1/24/08}

\usepackage{algorithm}
\usepackage{algpseudocode}
\usepackage{multicol}
\usepackage{multirow}
\usepackage{subfigure}
\usepackage{balance}
\usepackage{graphicx}
\acmSubmissionID{79}

\newtheorem{definition}{Definition}[section]
\newtheorem{theorem}{Theorem}[section]
\newtheorem{lemma}[theorem]{Lemma}

\begin{document}

\newcommand{\bs}{\boldsymbol}

\title{AIGB: Generative Auto-bidding via Diffusion Modeling}

\author{Jiayan Guo}
\authornote{Work is done during the internship at Alibaba Group.}
\email{guojiayan@pku.edu.cn}

\affiliation{%
  \institution{Peking University \\ Alibaba Group}
  \city{Beijing}
  \country{China}
}

\author{Yusen Huo}
\email{huoyusen.huoyusen@alibaba-inc.com}
\affiliation{%
  \institution{Alibaba Group}
  \city{Beijing}
  \country{China}}

\author{Zhilin Zhang}
\email{zhangzhilin.pt@alibaba-inc.com}
\affiliation{%
  \institution{Alibaba Group}
  \city{Beijing}
  \country{China}
}

\author{Tianyu Wang}
\email{yves.wty@@alibaba-inc.com}
\affiliation{%
 \institution{Alibaba Group}
 \city{Beijing}
 \country{China}}

\author{Chuan Yu}
\email{yuchuan.yc@alibaba-inc.com}
\affiliation{%
 \institution{Alibaba Group}
 \city{Beijing}
 \country{China}}

\author{Jian Xu}
\email{xiyu.xj@alibaba-inc.com}
\affiliation{%
 \institution{Alibaba Group}
 \city{Beijing}
 \country{China}}

  \author{Yan Zhang}
\email{zhyzhy001@pku.edu.cn}
\affiliation{%
  \institution{Peking University}
  \city{Beijing}
  \country{China}
} 

 \author{Bo Zheng}
\email{bozheng@alibaba-inc.com}
\authornote{Bo Zheng is the corresponding author.}
\affiliation{%
 \institution{Alibaba Group}
 \city{Beijing}
 \country{China}}

\renewcommand{\shortauthors}{Jiayan and Yusen, et al.}

\begin{abstract}
Auto-bidding plays a crucial role in facilitating online advertising by automatically providing bids for advertisers. Reinforcement learning~(RL) has gained popularity for auto-bidding. However, most current RL auto-bidding methods are modeled through the Markovian Decision Process~(MDP), which assumes the Markovian state transition. This assumption restricts the ability to perform in long horizon scenarios and makes the model unstable when dealing with highly random online advertising environments. To tackle this issue, this paper introduces AI-Generated Bidding (AIGB), a novel paradigm for auto-bidding through generative modeling. In this paradigm, we propose DiffBid, a conditional diffusion modeling approach for bid generation. DiffBid directly models the correlation between the return and the entire trajectory, effectively avoiding error propagation across time steps in long horizons. Additionally, DiffBid offers a versatile approach for generating trajectories that maximize given targets while adhering to specific constraints. Extensive experiments conducted on the real-world dataset and online A/B test on Alibaba advertising platform demonstrate the effectiveness of DiffBid, achieving 2.81\% increase in GMV and 3.36\% increase in ROI.
\end{abstract}





\maketitle

\section{Introduction}

The ever-increasing digitalization of commerce has exponentially expanded the scope and importance of online advertising platforms~\cite{ha2008online,evans2009online}. These ad platforms have become indispensable for businesses to effectively target their audience and drive sales. Traditionally, advertisers need to manually adjust bid prices to optimize overall ad performance. However, this coarse bidding process becomes impractical when dealing with trillions of impression opportunities, requiring extensive domain knowledge~\cite{chiesi1979acquisition} and comprehensive information about the advertising environments.

To alleviate the burden of bid optimization for advertisers,  these ad platforms provide auto-bidding services~\cite{deng2021towards,balseiro2021landscape,balseiro2021robust,ou2023deep}. These services automate the determination of bids for each impression opportunity by employing well-designed bidding strategies. Such strategies consider a variety of factors about advertising environments and advertisers, such as the distribution of impression opportunities, budgets, and average cost constraints~\cite{li2022auto}. Considering the dynamic nature of advertising environments, it is essential to regularly optimize the bidding strategy, typically at intervals of a few minutes, in response to changing conditions. 
With advertising episodes typically extending beyond 24 hours, auto-bidding can be seen as a sequential decision-making process with a long planning horizon where the bidding strategy seeks to optimize performance throughout the entire episode.

Recently, reinforcement learning~(RL) techniques have been employed to optimize auto-bidding strategies
through the training of agents with bidding logs collected from online advertising environments~\cite{jin2018real,cai2017real,wang2017display,he2021unified,zhang2023personalized,mou2022sustainable}. By leveraging historical realistic bidding information, these agents can learn patterns and trends to make informed bidding decisions. However, most existing RL auto-bidding methods are based on the Markovian decision process~(MDP), where the next state only depends on the current state and action. In the online auto-bidding environment, this assumption may been challenged by our statistical analysis presented in Figure~\ref{fig:corr}, which shows a significant increase in the correlation between the sequence lengths of history states and the next bidding state. This result indicates that solving auto-bidding considering only the last state will encounter several problems, including instability in the highly random online advertising environment. Additionally, the RL methods that rely on the Bellman equation often result in compound errors ~\cite{pmlr-v162-fujimoto22a}. This issue is especially pronounced in the auto-bidding problem characterized by sparse return and limited data coverage. A detailed statistical analysis is provided in~\ref{sec:action_control}.
\begin{figure}[t]
    \centering
    \includegraphics[width=.6\linewidth]{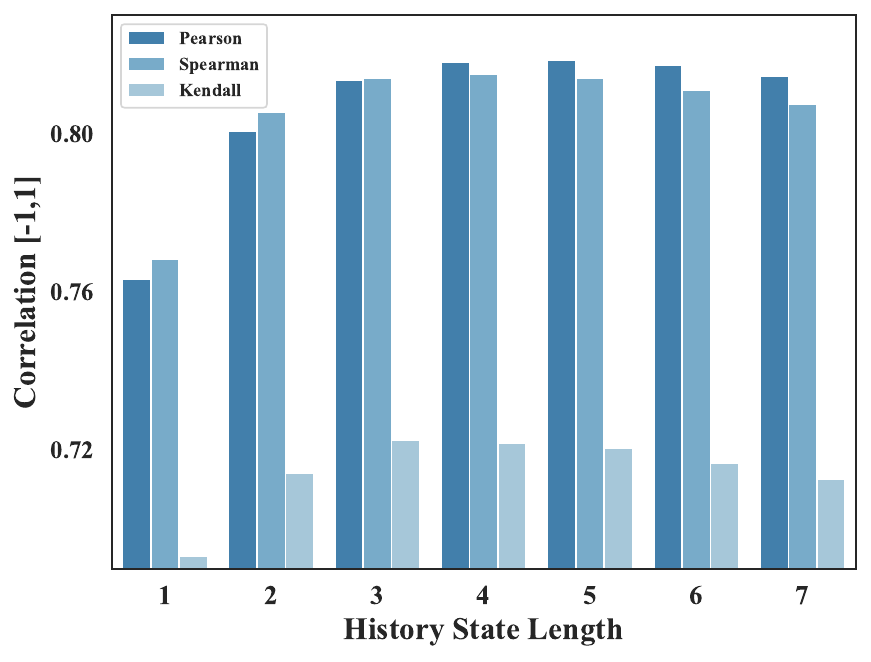}
    \caption{Correlation Coefficients between History and the Next State.}
    \label{fig:corr}
    \vspace{-0.2cm}
\end{figure}
In this paper, instead of employing RL-based methods, we present a novel paradigm, AI Generated Bidding (AIGB), that regards auto-bidding as a generative sequential decision-making problem. AIGB directly capture the correlation between the return and the entire bidding trajectory that consists of a sequence of states or actions, thereby transforming the problem into learning to generate an optimal bidding trajectory. This approach enables us to overcome the limitations of RL when dealing with the highly random online advertising environment, sparse returns, and limited data coverage.

In the new paradigm, we propose \textbf{\textit{Diff}}usion auto-\textbf{\textit{bid}}ding model DiffBid. It gradually corrupts the bidding trajectory by injecting scheduled Gaussian noises into the forward process. Then, it reconstructs trajectory from corrupted ones given returns and temporal conditions via a parameterized neural network. 
We further propose a non-Markovian inverse dynamics~\cite{nguyen2008learning,10.1145/3488560.3498524,10.1145/3539597.3570445} to more accurately generate optimal bidding parameters. Taking one step further, DiffBid provides flexibility to closely align with the specific needs of advertisers by accommodating diverse constraints like cost-per-click (CPC) and incorporating human feedback. Notably, DiffBid serves as a unified model capable of mastering multiple tasks simultaneously, dynamically composing various bidding trajectory components to generate sequences that efficiently maximize diverse targets while adhering to a range of predefined constraints. To assess the effectiveness of DiffBid, we conducted extensive evaluations offline and online against baselines. Our results indicate that DiffBid surpasses RL methods for auto-bidding. 
In summary:
\begin{itemize}
\item 
We uncover that the Markov assumptions upon which common decision-making methods rely are not applicable to the auto-bidding problem. Therefore, we propose a novel bidding paradigm with non-Markovian properties based on generative learning. This paradigm represents a significant innovation in modeling methodology compared with existing RL methods commonly used in auto-bidding.


\item 
Unlike common bidding methods, our approach captures the correlation between the return and the entire bidding trajectory. This design enables the method to address important challenges, such as sparse returns, and ensures stability in the highly random advertising environment. Finally, we prove that the proposed diffusion modeling is equivalent in terms of optimality to solving a non-Markovian decision problem.
\item We demonstrate that the method can integrate capabilities to handle a variety of tasks within a unified solution, transcending the limitations of traditional task-specific methods. It shows that DiffBid outperforms conventional RL methods in auto-bidding, and achieves significant performance gain on a leading E-commerce ad platform through both offline and online evaluation. In specific, it achieves 2.81\% increase in GMV and 3.36\% in ROI.
\end{itemize}


\section{Preliminary}

\subsection{Problem Formulation}
\begin{figure*}[t]
    \centering
    \includegraphics[width=.75\linewidth]{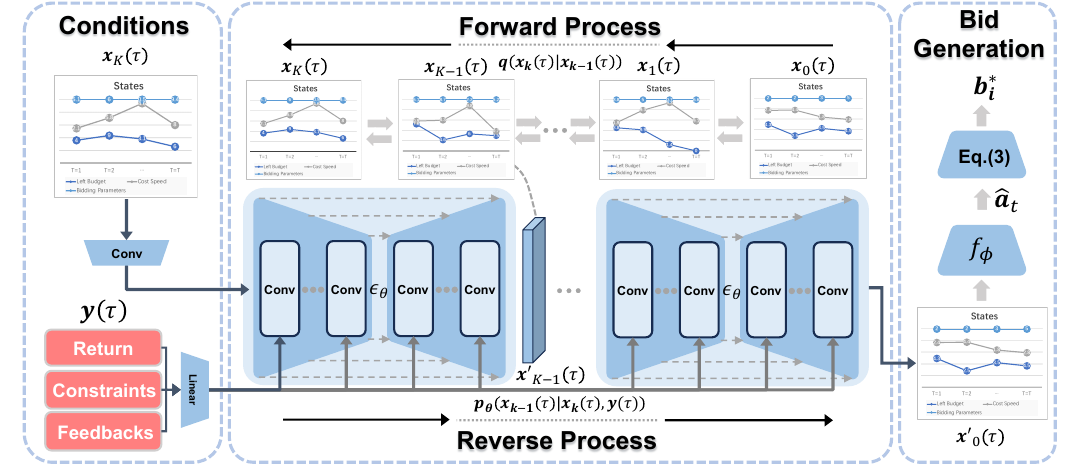}
    \caption{Overall Framework for Generative Auto-bidding.} 
    
    \label{fig:model}
\end{figure*}
For simplicity, we consider auto-bidding with cost-related constraints. During a time period, suppose there are $N$ impression opportunities arriving sequentially and indexed by $i$. In this setting, advertisers submit bids to compete for each impression opportunity. An advertiser will win the impression if its bid $b_i$ is greater than others. Then it will incur a cost $c_i$ for winning and getting the value. 

During the period, the mission of an advertiser is to maximize the total received value $\sum_io_iv_i$, where $v_i$ is the value of impression $i$ and $o_i$ is whether the advertiser wins impression $i$. Besides, we have the budget and several constraints to control the performance of ad deliveries. Budget constraints are simply $\sum_{i}o_ic_i\le B$, where $c_i$ is the cost of impression $i$ and $B$ is the budget. The other constraints are complex and according to~\cite{he2021unified} we have the unified formulation:
\begin{equation}
    \frac{\sum_{i}c_{ij}o_i}{\sum_i p_{ij}o_i}\le C_j,
\end{equation}
\noindent where $C_j$ is the upper bound of $j$'th constraint. $p_{ij}$ can be any performance indicator, e.g. return, or constant. $c_{ij}$ is the cost of constraint $j$. Given $J$ constraints, we have the Multi-constrained Bidding~(MCB) as:
\begin{equation}
\begin{split}
    \mathop{\text{maximize}}_{o_i} & \sum_{i}o_iv_i \\
    \text{s.t.} & \sum_{i}o_ic_i\le B \\
    & \frac{\sum_{i}c_{ij}o_i}{\sum_i p_{ij}o_i}\le C_j, \ \ \ \forall j \\
    & \ \ \ \ \  \ \ \ \ \ \ o_i\in\{0,1\}, \ \forall i \\
\end{split}
\label{eq:problem}
\end{equation}
A previous study~\cite{he2021unified} has already shown the optimal solution:
\begin{equation}
    b_i^*=\lambda_0 v_i+C_i\sum_{j=1}^J\lambda_j p_{ij},
    \label{eq:bid_formula}
\end{equation}
\noindent where $b_i^*$ is the predicted optimal bid for the impression $i$. $\lambda_j, \ j \in \{0,...,J\}$ are the optimal bidding parameters. Specifically, when considering only the budget constraint, it is the Max Return bidding. However, when considering both the budget constraint and the CPC constraint, it is called Target-CPC bidding.
From an alternative perspective, the optimal strategy involves arranging all impressions in order of their cost-effectiveness (CE) and then selecting every impression opportunity that surpasses the optimal CE ratio $ce^*$. This threshold enforces the constraint, and the optimal bidding parameters $\lambda_0  = 1 / ce^*$. 

\subsection{Auto-Bidding as Decision-Making}
Eq.(\ref{eq:bid_formula}) gives the formation of the optimal bid $b_i^*$ with bidding parameters $\lambda_j, \ j \in \{0,...,J\} $. However, in practice, the highly random and complex nature of the advertising environment prevents direct calculation of the bidding parameters. They must be carefully calibrated to adapt to the environment and dynamically adjusted as it evolves, This subsequently makes auto-bidding a sequential decision-making problem. To model it with decision-making, we introduce states $\bs{s}_t\in\mathcal{S}$ to describe the real-time advertising status and actions $\bs{a}_t\in\mathcal{A}$ to adjust the corresponding bidding parameters. The auto-bidding agent will take action $\bs{a}_t$ at the state $\bs{s}_t$ based on its policy $\pi$, and then the state will transit to the next state $\bs{s}_{t+1}$ and gain reward $\bs{r}_t\in\mathcal{R}$ according to the advertising environment dynamics $\mathcal{T}$. When $\mathcal{T}:\bs{s}_t\times \bs{a}_t\rightarrow\bs{s}_{t+1}\times\bs{r}_t$ satisfies, it is called the Markovian decision process~(MDP). Otherwise, it is a non-Markovian decision process. We next describe the key items of the automated 
bidding agent in the industrial online advertising system:
\begin{itemize}
    \item State $\bs{s}_t$ describes the real-time advertising status at time period $t$, which includes 1) remaining time of the advertiser; 2) remaining budget; 3) budget spend speed; 4) real-time cost-efficiency~(CPC), 5) and average cost-efficiency~(CPC). 
    \item Action $\bs{a}_t$ indicates the adjustment to the bidding parameters at the time period $t$, which has the dimension of the number of bidding parameters $\lambda_j, \ j=1,..,J$ and modeled as $(a_t^{\lambda_0},...,a_t^{\lambda_J})$. 
    \item The reward $\bs{r}_t$ is the value contributed to the objective obtained within the time period $t$.
    \item A trajectory $\tau$ is the index of a sequence of states, actions, and rewards within an episode.
\end{itemize}
In the online advertising system, learning policy through direct interaction with the online environment is unfeasible due to safety concerns. Nonetheless, access to historical bidding logs, incorporating trajectories from a variety of bidding strategies, is attainable and provides a viable alternative. Prevalent auto-bidding methods predominantly leverage this offline data to craft effective policies. Our approach is aligned with this practice and will be elaborated in detail in the subsequent chapters.

\section{AIGB PARADIGM for AUTO-BIDDING }


To thoroughly investigate the auto-bidding problem, we conducted a series of statistical analyses of bidding trajectories, with detailed information available in appendix~\ref{statistical_analysis}. These analyses provide us with the insight that devising an effective bidding strategy is essentially equivalent to optimizing a state trajectory. Armed with this insight, we propose a hierarchical paradigm for auto-bidding that prioritizes the state trajectory optimization and subsequently generates actions aligned with the optimized trajectory.

For state trajectory optimization, we can employ a generative model to capture the joint distribution of the entire bidding trajectory and its associated returns, subsequently generating the trajectory distribution conditioned on the desired return. This approach enables us to address key auto-bidding challenges by employing SOTA generative algorithms. This paper presents an implementation that utilizes Denoising Diffusion Probabilistic Models (DDPM). For action generation, several off-the-shelf methods can be utilized to predict the proper action given the target state trajectory. In this paper, we apply a widely used inverse dynamics model.
The hierarchical paradigm divides auto-bidding into two supervised learning problems, offering several advantages that include enhanced interpretability and increased stability during the training process. 

\section{Diffusion Auto-bidding Model}

In this section, we give a detailed introduction of the proposed diffusion Auto-bidding Model~(DiffBid). We will first give the modeling of Auto-bidding through diffusion models in Section~\ref{sec:overview}. Then we give a detailed description of the forward process in Section~\ref{sec:forward}, the reverse process in Section~\ref{sec:generation}, and the training process in Section~\ref{sec:training}. Finally, we will give the complexity analysis in Section~\ref{sec:complexity}.



\subsection{Diffusion Modeling of Auto-bidding}
\label{sec:sequential_modeling}

\subsubsection{\textbf{Overview}}
\label{sec:overview}
We model such sequential decision-making problem through conditional generative modeling~\cite{chen2021decision,ajay2022conditional} by maximum likelihood estimation~(MLE):
\begin{equation}
    \mathop{\text{max}}_\theta\mathbb{E}_{\tau\sim D}\left [ \text{log}p_\theta (\boldsymbol{x}_0(\tau)|\bs{y}(\tau)) \right ]
\label{eq:objective}
\end{equation}
\noindent where $\tau$ is the trajectory index, $\bs{x}_0(\tau)$ is the original trajectory of states and $\bs{y}(\tau)$ is the corresponding property. The goal is to estimate the conditional data distribution with $p_\theta$ so that the future states of a trajectory $\bs{x}_0(\tau)$ from information $\bs{y}(\tau)$ can be generated. For example, in the context of online advertising, $\bs{y}(\tau)$ can be the constraints or the total value of the entire trajectory. Under such a setting, we can formalize the conditional diffusion modeling for auto-bidding:
\begin{equation}
    q(\bs{x}_{k+1}(\tau)|\bs{x}_k(\tau)), \ \ p_\theta(\bs{x}_{k-1}(\tau)|\bs{x}_k(\tau),\bs{y}(\tau)),
\end{equation}
\noindent where $q$ represents the forward process in which noises are gradually added to the trajectory while $p_\theta$ is the reverse process where a model is used for denoising. The detailed introduction of diffusion modeling can be found in Appendix~\ref{sec:diffusion_modeling}. The overall framework is presented in Figure~\ref{fig:model}. We will make a detailed discussion about the two modeling processes in the following sections.

\subsubsection{\textbf{Forward Process via Diffusion over States}}
\label{sec:forward}


We model the forward process $q(\bs{x}_{k+1}(\tau)|\bs{x}_{k}(\tau))$ via diffusion over states, where:
\begin{equation}
    \bs{x}_{k}(\tau):=\left (\bs{s}_1,...,\bs{s}_t,...,\bs{s}_{T} \right )_k,
\end{equation}
where $\bs{s}_t$ is modeled as a one-dimensional vector. $\bs{x}_{k}(\bs{\tau})$ is a noise sequence of states and can be represented by a two-dimensional array where the first dimension is the time periods and the second dimension is the state values. Merely sampling states is not enough for an agent. Given $\bs{x}_k(\tau)$, we model the diffusion process as a Markov chain, where $\bs{x}_k(\tau)$ is only dependent on $\bs{x}_{k-1}(\tau)$:
\begin{equation}
    q(\bs{x}_k(\tau)|\bs{x}_{k-1}(\tau))=\mathcal{N}\left ( \bs{x}_{k}(\tau);\sqrt{1-\beta_k}\bs{x}_{k-1}(\tau),\beta_k I \right ),
\label{eq:add_noise}
\end{equation}  
\noindent when $k\rightarrow\infty$, $x_k(\tau)$ approaches a sequence of standard Gaussian distribution where we can make sampling through the re-parameterization trick and then gradually denoise the trajectory to produce the final state sequence. For the design of $\beta_k, \ \ k=1,...,K$, we apply cosine schedule~\cite{nichol2021improved} to assign the corresponding values which smoothly increases diffusion noises using a cosine function to prevent sudden changes in the noise level. The details for noise schedule can be found in the appendix.


\subsubsection{\textbf{Reverse Process for Bid Generation}}
\label{sec:generation}



Following~\cite{ho2021classifier,ajay2022conditional} we use a classifier-free guidance strategy with low-temperature sampling to guide the generation of bidding, to extract high-likelihood trajectories in the dataset. During the training phase, we jointly train the unconditional model $\epsilon_\theta(\bs{x}_{k}(\tau),k)$ and conditional model $\epsilon_\theta(\bs{x}_{k}(\tau),\bs{y}(\tau),k)$ by randomly dropping out conditions. During generation, a linear combination of conditional and unconditional score estimates is used:
\begin{equation}
\hat{\epsilon}_k:=\epsilon_\theta(\bs{x}_k(\tau),k)+\omega\left( \epsilon_\theta\left( \bs{x}_k(\tau),\bs{y}(\tau),k \right) - \epsilon_\theta \left ( \bs{x}_k(\tau), k \right) \right),
\label{eq:perturbed_noise}
\end{equation}
\noindent where the scale $\omega$ is applied to extract the most suitable portion of the trajectory in the dataset that coappeared with $\bs{y}(\tau)$. After that, we can sample from DiffBid to produce bidding parameters through sampling from $p_\theta(\bs{x}_{k-1}(\tau)|\bs{x}_{k}(\tau),\bs{y}(\tau))$:
\begin{equation}
    \bs{x}_{k-1}(\tau)\sim\mathcal{N}\left (\bs{x}_{k-1}(\tau)|\bs{\mu}_{\theta}\left(x_k(\tau),\bs{y}(\tau),k\right),\bs{\Sigma}_{\theta}\left(\bs{x}_k(\tau),k\right ) \right )
\end{equation}
\noindent where a widely used parameterization here is  $\bs{\mu}_\theta(\bs{x}_k(\tau),\bs{y}(\tau),k)=\frac{1}{\sqrt{\alpha_k}}(\bs{x}_k(\tau)-\frac{\beta_k}{\sqrt{1-\overline{\alpha}_k}}\hat{\epsilon}_k )$ and $\Sigma_\theta(\cdot)=\beta_k$, in which $\alpha_k=1-\beta_k$ and $\overline{\alpha_k}=\prod_{i=1}^k\alpha_k$. When serving at time period $t$, the agent first sample a initial trajectory $x'_K(\tau)\sim\mathcal{N}(0,I)$ and assign the history states $\bs{s}_{0:t}$ into it. Then, we can sample predicted states with the reverse process recursively by
\begin{equation}
    \bs{x}'_{k-1}(\tau)=\bs{\mu}_\theta(\bs{x}'_k(\tau),\bs{y}(\tau), k)+\sqrt{\beta_k}\bs{z}
    \label{eq:sample}
\end{equation}
\noindent where $\bs{z}\sim\mathcal{N}(0,I)$. Given $\bs{x}'_0(\tau)$, we can extract the next predicted state $s'_{t+1}$, and determine how much it should bid to achieve that state. In this setting, we apply widely used inverse dynamics~\cite{agrawal2016learning,pathak2018zero} with non-Markovian state sequence to determine current bidding parameters at time period $t$:
\begin{equation}
    \bs{\hat{a}}_{t}=f_\phi(\bs{s}_{t-L:t},\bs{s}'_{t+1}),
\end{equation}
\noindent where $\bs{\hat{a}}_{t}\in\mathbb{R}^J$ contains predicted bidding parameters~(i.e. $\lambda_i, \ i=1,...,n$) at time $t$. $L$ is the length of history states. The inverse dynamic function $f_\phi$ can be trained with the same offline logs as the reverse process. This design disentangles the learning of states and actions, making it easier to learn the connection between states thus achieving better empirical performance. The overall procedure is summarized in Algorithm~\ref{algo:generation}.

\subsection{DiffBid Training}
\label{sec:training}

Following~\cite{ho2020denoising}, we train DiffBid to approximate the given noise and the returns in a supervised manner. Given a bidding trajectory $\bs{x}_0(\tau)$, we have its corresponding returns e.g., values the advertiser received, the constraint the model should obey and the history states $s_l,l=1,...,t$ before time $t+1$. Then we just train the reverse process model $p_\theta$ which is parameterized through the noise model $\epsilon_\theta$ and the inverse dynamics $f_\phi$ through:
\begin{equation}
\begin{split}
\mathcal{L}(\theta,\phi)&=\mathbb{E}_{k,\tau\in\mathcal{D}}\left [ ||\epsilon-\epsilon_\theta(\bs{x}_k(\tau),\bs{y}(\tau),k)||^2\right ] \\
&+\mathbb{E}_{(\bs{s}_{t-L:t},\bs{a}_t,s'_{t+1})\in\mathcal{D}}\left[ ||\bs{a}_t-f_\phi(\bs{s}_{t-L:t},\bs{s}'_{t+1})||^2 \right],
\end{split}
\label{eq:training}
\end{equation}
In the training process, we randomly sample a bidding trajectory $\tau$ and a time step $k$, then we construct a noise trajectory $\bs{x}_k(\tau)$ and predict the noise through Eq~(\ref{eq:perturbed_noise}). Following~\cite{ho2021classifier}, we randomly drop conditions $\bs{y}(\tau)$ with probability $p$ to train DiffBid to enhance the robustness. The process is presented in Algorithm~\ref{algo:training}. 

\subsection{Design of Conditions.}

In this section, we present approaches transforming industrial metrics into conditions of DiffBid.
\subsubsection{\textbf{Generation with Returns}} For each trajectory $\tau$ we have the total value the advertiser received as the the return $R(\tau)=\sum_{t=1}^T r_t$. We normalize the return by:
\begin{equation}
\label{eq:reward}
    R=\frac{R(\tau)-R_{\text{min}}}{R_{\text{max}}-R_{\text{min}}},
\end{equation}
\noindent where $R_{\text{min}}$ and $R_{\text{max}}$ are the smallest and the largest return in the dataset. Through Eq.~\ref{eq:reward} we normalize the return into $[0,1]$ and merge it into $y(\tau)$. Subsequently, we train the model to generate trajectories conditioned on the normalized returns. It should be noted that trajectories with more values received have higher normalized returns. Thus $R=1$ indicates the best trajectory with the highest values which will better fit the advertisers' needs. When generation, we just set $R=1$ and generate the trajectory under the max return condition to the advertiser. 


\subsubsection{\textbf{Generation with Constraints or Human Feedback}} In MCB, the cumulative performance related to the constraints within a given episode should be controlled so as not to exceed the advertisers' expectations. In such a setting, we can design $\bs{y}(\tau)$ to control the generation process. For example, in the Target-CPC setting, we can maintain a binary variable $E$ to indicate whether the final CPC exceeds the given constraint $C$:
\begin{equation}
    E = \text{I}_{x\le C}(x)
\label{eq:constraint}
\end{equation}
where $x=\frac{\sum_ic_{i}o_i}{\sum_ip_io_i}$ is defined in Eq~(\ref{eq:problem}). We can then normalize $x$ into $[0,1]$ through min-max normalization for simplification. $E$ can be used to indicate whether trajectory $\tau$ break the CPC constraint. We can also design $\bs{y}(\tau)$ to include $E=1$ to make the model generate bids that do not break the CPC constraint. Sometimes it is also important to adjust the bidding parameters given real-time feedback provided by the advertiser to enable flexibility. Here we use two example indicators that reflect the experience of advertisers:
\begin{enumerate}
    \item \textbf{Smoothness}: an advertiser may expect the cost curve as smooth as possible to avoid sudden change. By defining $x=\frac{1}{T}\sum_{t}\left|cost_{t}-cost_{t-1}\right|$, we can model it as a binary variable $S$ indicating whether the max cost change between adjacent time period exceeds a threshold as in Eq~(\ref{eq:constraint}).
    \item \textbf{Early/Late Spend}: an advertiser may expect the budget to be cost in the morning or in the evening when there are promotions. Here we model the ratio of cost in the early half day through $x=\frac{\sum_{t=0}^{T/2}{cost_i}}{\sum_{t=0}^{T}{cost_i}}$, and use a binary variable to indicate whether the spend in the early half day exceeds a certain threshold $C$ as in Eq~(\ref{eq:constraint}).
\end{enumerate}
We can also compose several constraints together to form $\bs{y}(\tau)$ to guide the model to generate bid parameters that adhere to different constraints. In this setting, $\bs{y}(\tau)$ will be a vector.


\subsection{Complexity Analysis}
\label{sec:complexity}
The complexity analysis for training DiffBid consists of the training process and the inference process. For training, given the time complexity of the noise prediction model $\epsilon_\theta$ is $\mathcal{O}(T_1)$, the complexity for the inverse dynamic model $f_\phi$ is $\mathcal{O}(T_2)$, the complexity for a training epoch is $\mathcal{O}(|\mathcal{B}|(T_1+T_2))$. It can be seen that the training complexity is linear with the input given $T_1$ and $T_2$ are relatively fixed. Thus the training of DiffBid is efficient. For generation, given the total diffusion step $K$, the trajectory length $L$, then the time complexity for inference is $\mathcal{O}(KL(T_1+T_2))$. We can observe that the time complexity for inference is linearly scaled with the diffusion step $K$. In image generation, $K$ is usually very large to ensure good generation quality, which brings the problem of non-efficiency. However, for bidding generation, we find $K$ needs not to be very large. Relatively small $K$ has already generated promising results. Moreover, in auto-bidding, a higher tolerance for latency is acceptable, enabling the use of relatively larger $K$.
\section{Theoretical Analysis}

In this section, we theoretically analyze the property of DiffBid. In specific, we show that DiffBid that utilize MLE as the objective has a corresponding non-Markovian decision problem~\cite{majeed2018q,qin2023learning,gaon2020reinforcement,mutti2022importance}. The detailed proofs can be found in the Appendix~\ref{ap:theory}.




\begin{lemma}[MLE as non-Markovian decision-making]
Assuming the Markovian transition $p_{\gamma^*}(s_{t+1}|s_t, a_t)$ is known, the ground-truth conditional state distribution $p^*(s_{t+1} |s_{0:t})$ for demonstration sequences is accessible, we can construct a non-Markovian sequential decision-making problem, based on a reward function $r_{\alpha}(s_{t+1},s_{0:t}):={\rm{log}}\int p_{\alpha}(a_t|s_{0:t})p_{\gamma^*}(s_{t+1}|s_t,a_t)d a_t$ for an arbitrary energy-based policy $p_\alpha(a_t|s_{0:t})$. Its objective is
\begin{equation}
    \sum_{t=0}^T\mathbb{E}_{p^*(s_{0:t})}\left [ V^{p_\alpha}(s_{0:t}) \right ]=\mathbb{E}_{p^*(s_{0:T})}\left [ \sum_{t=0}^{T}\sum_{k=t}^T r_{\alpha}(s_{k+1};s_{0:k}) \right ]
\end{equation}
$V^{p_\alpha}(s_{0:t}):=\mathbb{E}_{p^*(s_{t+1:T}|s_{0:t})}[\sum_{k=1}^Tr_\alpha(s_{t+1};s_{0:t})]$ is the value function of $p_\alpha$. This objective yields the save optimal policy as the Maximum Likelihood Estimation  $\mathbb{E}_{p^*(s_{0:T})}\left [{\rm{log}} p_\theta(s_{0:T})\right ]$.
 \end{lemma}

\noindent\textbf{Remarks.} This analysis shows that DiffBid utilizing MLE objective has its corresponding non-Markovian decision problem, and their optimal are equivalent. It means that DiffBid does not require the MDP assumption of problems and thus is more powerful in handling randomness and sparse return like in the advertising environment.
\section{Experiments}

\begin{table*}[t]
    \centering
    \caption{Performance Comparison with baselines in different settings, including different data scales, and budgets in Max Return bidding. \textit{improv} indicates the relative improvement of DiffBid against the most comparative baseline. The best results are bolded and the best second results are underlined.}
    \begin{tabular}{c|c|ccccccc}
    \toprule
    \textbf{Training Dataset}  & \textbf{Budget} & \textbf{USCB} & \textbf{BCQ} & \textbf{CQL} & \textbf{IQL} & \textbf{DT} & \textbf{DiffBid} & \textit{improv} \\
    \midrule
       \multirow{4}{*}{USCB-5K} & 1.5K & 454.25 & 454.72 & 461.82 & 456.80 & \underline{477.39} & \textbf{480.76} & 0.71\% \\
       & 2.0K & 482.67 & 483.50 & 475.78 & 486.56 & \underline{507.30} & \textbf{511.17} & 0.76\% \\
       & 2.5K & 497.66 & 498.77 & 481.37 & 518.27 & \underline{527.88} & \textbf{531.29} & 0.65\% \\
       & 3.0K & 500.60 & 501.86 & 491.36 & 549.19 & \underline{550.66} & \textbf{556.32} & 1.03\% \\
    \midrule
       \multirow{4}{*}{USCBEx-5K}  & 1.5K & 454.25 & 453.74 & 358.43 & \underline{464.69} & 378.64 & \textbf{475.62} & 2.35\% \\
       & 2.0K & 482.67 & 487.63 & 356.80 & \underline{529.36} & 439.03 &  \textbf{544.38} & 2.84\% \\
       & 2.5K & 497.66 & 510.75 & 356.41 & \underline{613.67} & 505.43 &  \textbf{624.29} & 1.73\% \\
       & 3.0K & 500.60 & 512.18 & 355.42 & \underline{670.65} & 574.79 &  \textbf{678.73} & 1.17\% \\
    \midrule
       \multirow{4}{*}{USCBEx-50K} & 1.5K & 454.25 & \underline{458.64} & 435.06 & 446.23 & 396.24 & \textbf{495.57} & 8.05\% \\
       & 2.0K & 482.67 & 491.72 & 431.49 & \underline{533.58} & 478.29 &  \textbf{551.73} & 3.40\% \\
       & 2.5K & 497.66 & 513.23 & 428.39 & \underline{592.32} & 554.48 &\textbf{606.34} & 2.37\% \\
       & 3.0K & 500.60 & 526.21 & 425.29 & \underline{633.26} & 611.50 & \textbf{644.88} & 1.83\% \\
    \bottomrule
    \end{tabular}
    \label{tab:main_exp}
\end{table*}

\subsection{Experimental Setup}

\subsubsection{\textbf{Experimental Environment}} The simulated experimental environment is conducted in a manually built offline real advertising system~(RAS) as in~\cite{mou2022sustainable}.  Specifically, the RAS is composed of two consecutive stages, where the auction mechanisms resemble those in the RAS. We consider the bidding process in a day, where the episode is divided into 96 time steps. Thus, the duration between any two adjacent time steps $t$ and $t + 1$ is 15 minutes. The number of impression opportunities between time step $t$ and $t + 1$ fluctuates from 100 to 500. Detailed parameters in the RAS are shown in Table~\ref{tab:environment}. We keep the parameters the same for all experiments.

\subsubsection{\textbf{Data Collection}} We use the widely applied auto-bidding RL method USCB in the online environment to generate the bidding logs for offline RL training. This results in a total $5,000$ trajectories for the based dataset and $50,000$ for a larger one. To increase the diversity of the action space, we also randomly make explorations to generate a dataset with more noise. The above process results in three datasets: USCB-5k, USCBEx-5k, and USCBEx-50k, where USCBEx indicates USCB logs with random exploration data.

\subsubsection{\textbf{Baselines.}} We use the state-of-the-art auto-bidding method USCB as well as other 4 recently proposed offline RL methods as our baselines.  The details of the baselines are as follows:

\begin{itemize}
    \item \textbf{USCB}~\cite{he2021unified} an RL method designed for real-time bidding to dynamically adjust parameters to achieve the optimum. It has outperformed many RL baselines and is also the base policy that is used to collect the data for offline training.
    \item \textbf{BCQ}~\cite{fujimoto2019off} a classic offline RL method without interaction with the environment. 
    \item \textbf{CQL}~\cite{kumar2020conservative} address the limitations of offline RL methods by learning a conservative Q-function such that the expected value of a policy under it lower-bounds its true value.
    \item \textbf{IQL}~\cite{kostrikov2021offline} an offline RL method that does not require evaluating actions outside of the dataset, yet it enables substantial improvement of the learned policy beyond the best behavior in the data through generalization
    \item \textbf{DT}~\cite{chen2021decision} a prevalent generative method based on the transformer architecture for sequential decision-making.
\end{itemize}

\subsubsection{\textbf{Implementation Details}} For the implementation of baselines, we use the default hyper-parameters suggested from their papers and also tune through our best effort. For DiffBid, the diffusion steps is searched within $\{5, 10, 20, 30, 50\}$. $\gamma$ is set to 0.008. $L$ is searched in $\{1,2,3\}$. $\omega$ for noise schedule is set to 0.2 empirically. The batch size is set to 2$\%$ of all training trajectories. Total training epochs is set to 500. For the implementation of $p_\theta$, we adopt the most widely used model U-Net for diffusion modeling with hidden sizes of 128 and 256. We use Adam optimizer with a learning rate $1e{-4}$ to optimize the model. The condition dropout ratio is set to 0.2 during training. We update the model with momentum updates over a period of 4 steps. 

\subsubsection{\textbf{Evaluation}} For evaluation, we randomly initialize a multi-agent advertising environment with USCB as the base auto-bidding agents and use other methods to compete with these agents. We test the performance under 4 different budgets, 1500, 2000, 2500, and 3000, to test the generalization under different budget scales. We use the cumulative reward as the evaluation metric, which reflects the total gain received by the target agent. For each method, we randomly initialize 50 times and report the average of top-5 scores.

\subsection{Performance Evaluation}

The performance against baselines is shown in Table~\ref{tab:main_exp}. In this table, we show the cumulative reward from different budgets of all the models. We have the following discoveries. One of the key takeaways from the performance comparison presented in Table~\ref{tab:main_exp} is that offline RL methods consistently outperform the state-of-the-art auto-bidding method, USCB. This finding underscores the advantages of leveraging historical bidding data to train RL agents. Offline RL methods, such as BCQ, IQL, and DT, exhibit superior performance in terms of cumulative rewards across various budget scenarios. The superiority of offline RL methods can be attributed to their ability to learn from past bidding experiences without interaction with a simulation environment. This mitigates the challenges associated with inconsistencies between the online bidding environment and the offline bidding environment, leading to policies that are better aligned with real-world scenarios. Notably, DiffBid stands out as the top-performing approach among all the methods evaluated. In all budget scenarios and training datasets, DiffBid consistently achieves the highest cumulative rewards. This remarkable performance highlights the efficacy of the DiffBid approach in optimizing bidding strategies by directly modeling the correlation with the returns and entire trajectories. By decoupling the computational complexity from horizon length, DiffBid achieves superior decision-making capabilities, outperforming traditional RL methods in both foresight and strategy.
Another important observation from the results is the impact of training dataset size on model performance. When comparing the "USCB-5K" and "USCBEx-50K" settings, it becomes evident that a larger training dataset consistently leads to improved cumulative rewards. This finding underscores the significance of data size in training RL models for automated bidding. A richer dataset allows the models to capture more diverse bidding scenarios and make more informed decisions, ultimately resulting in better performance. One intriguing aspect of DiffBid's performance is its resilience to noise. In real-world advertising environments, there can be inherent uncertainty and variability in the bidding process due to factors like market dynamics and competitor behavior. DiffBid appears to handle such noise more effectively than the RL baselines. This means that even in situations where bidding outcomes are less predictable, DiffBid manages to maintain competitive performance. 



\subsection{Ablation Study}
\begin{table}[t]
    \centering
    \caption{Ablation Study}
\resizebox{.8\linewidth}{!}{
    \begin{tabular}{ccc}
    \toprule
       Model & USCBEx-5K & USCBEx-50K \\
    \midrule
    DiffBid &  2280.12 & 2395.60 \\
    \midrule
      DiffBid w/o cond  & 1812.64 & 1852.21 \\
      DiffBid w/o non-mkv &   2254.78  & 2287.41 \\
    \bottomrule
    \end{tabular}}
    \label{tab:ablation}
\end{table}
To study different parts of the proposed DiffBid, we run the model without a certain module to see if the removed corresponding module will result in a performance drop. The result of the ablation study is shown in Table~\ref{tab:ablation}. Due to the space limitation, we only provide the results on USCBEx-5K and USCBEx-50K. w/o cond refers to the DiffBid with the condition set to 0.0~(rather than 1.0). w/o non-mkv refers to the situation where we only use the current state and the predicted next state to generate the bidding coefficient. From the table, we find both of the two parts contribute to the final result, and removing either of them will result in a performance drop. It verifies the effectiveness of the proposed methods in boosting DiffBid's performance for auto-bidding.

\subsection{In-depth Analysis}

\subsubsection{\textbf{Study of State Transition}}

Here we compare the state transition of the baseline method USCB and our proposed method DiffBid. The result for grouped and non-grouped state transition during a day is shown in Figure~\ref{fig:state_transition}. In this figure, we plot the budget left ratio with time steps in one day. From the figure, we can observe that under USCB, most of the advertisers' consumption does not exhaust their budget. This is attributed to the inconsistency between the offline virtual environment and the real online environment faced by USCB. On the contrary, the budget completion situation improves under DiffBid, where most of the advertisers spend more than 80\% of their budgets. One possible reason is that DiffBid finds trajectories with a high budget completion ratio will also have a high cumulative reward, and thus tend to generate trajectories with a high budget completion ratio. Moreover, advertisers with small budgets undertend to spend money in the afternoon. This is because the impressions in the afternoon offer a higher cost-effectiveness, albeit with a limited quantity. 


\subsubsection{\textbf{Performance under Constraints and Feedbacks.}}
\begin{figure}[t]
    \centering
    \subfigure[USCB]{\includegraphics[width=.495\linewidth]{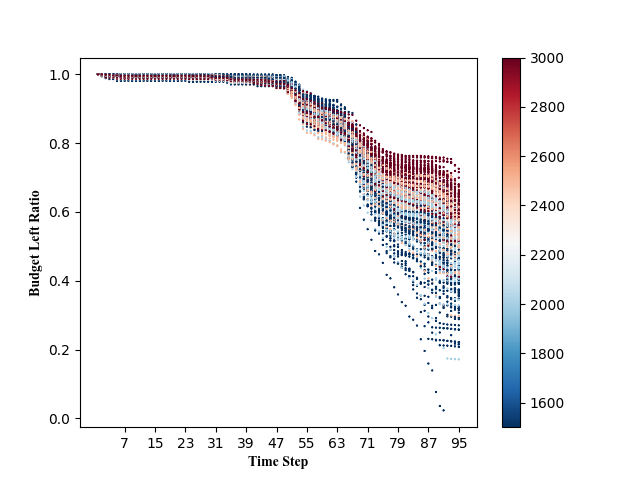}}
    \subfigure[DiffBid]{\includegraphics[width=.495\linewidth]{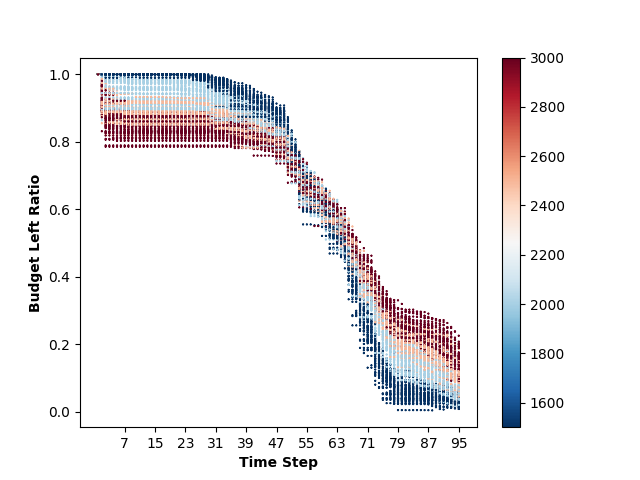}\label{fig:stable}}
    \caption{State Transition in One Episode.}
    \label{fig:state_transition}
\end{figure}
\begin{figure}[t]
    \centering
    \subfigure[IQL]{\includegraphics[width=.495\linewidth]{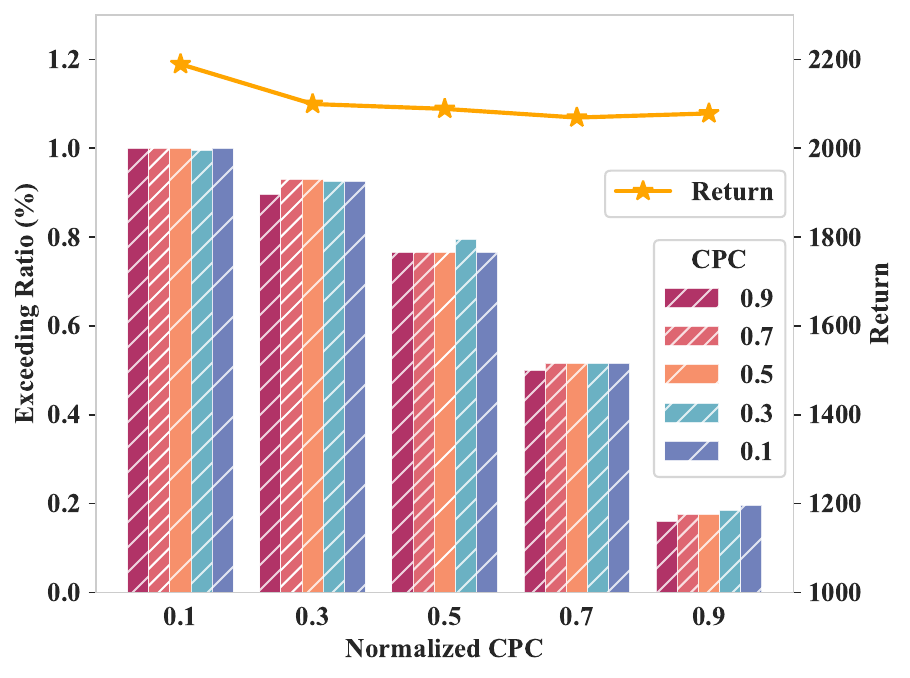}}
    \subfigure[DiffBid]{\includegraphics[width=.495\linewidth]{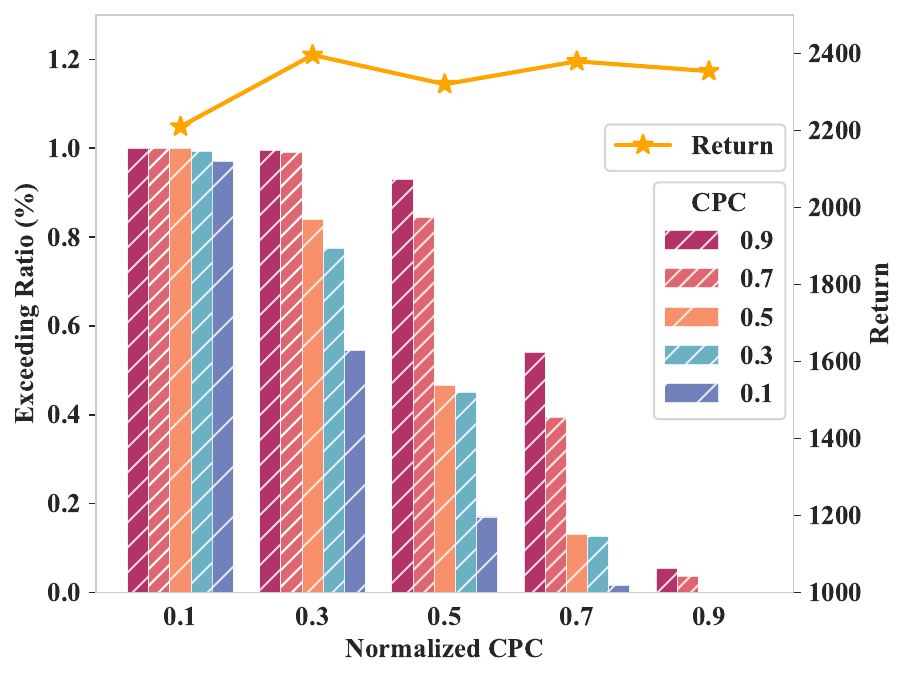}}
    \caption{Performance under CPC constraint.}
    \label{fig:ppc_constraint}
\end{figure}
\begin{figure}[t]
    \centering
    \subfigure[Smoothness]{\includegraphics[width=.50\linewidth]{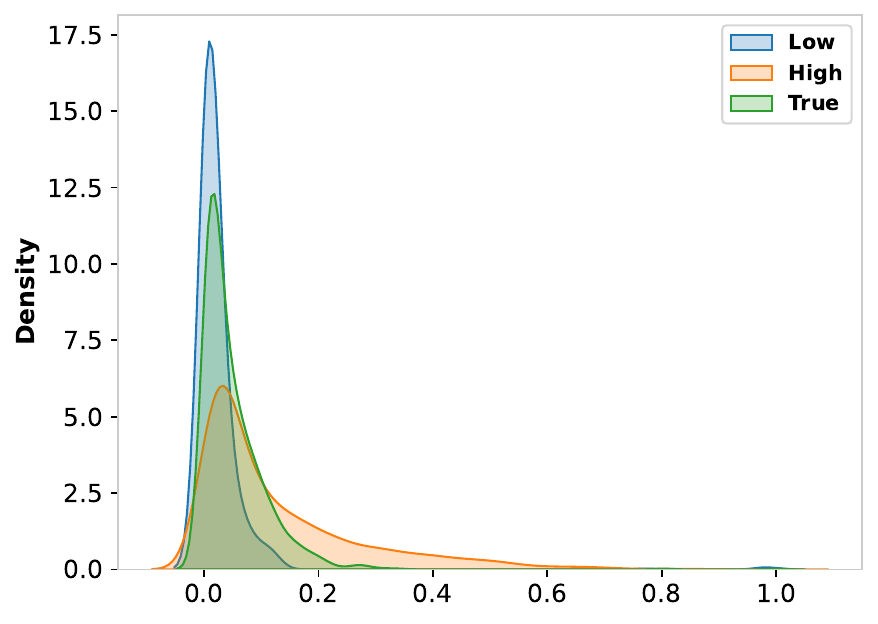}}
    \subfigure[Early Spend]{\includegraphics[width=.49\linewidth]{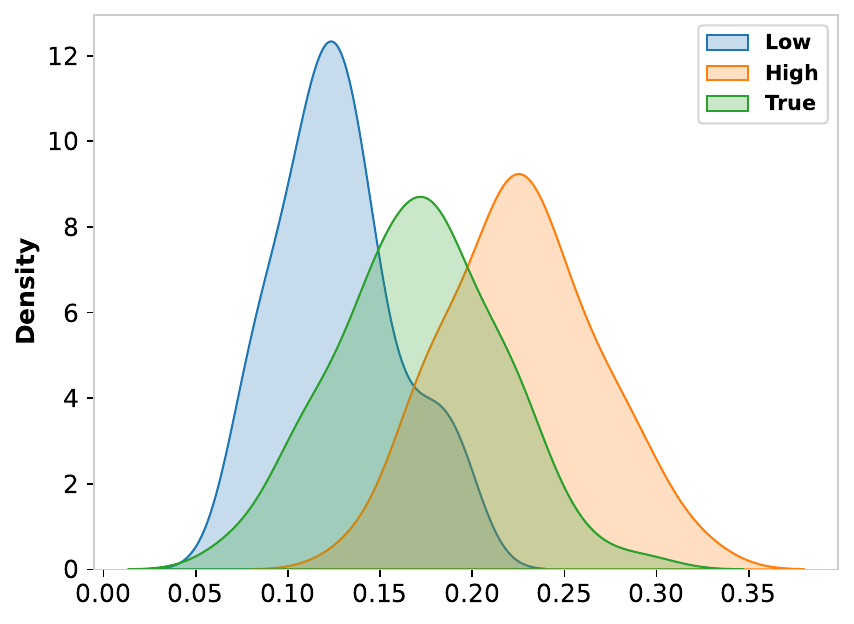}}
    \caption{Performance of Human Feedback.}
    \label{fig:feedback}
\end{figure}
We additionally investigate DiffBid's multi-objective optimization capability under specific constraints, comparing its performance with Offline RL. Specifically, we choose CPC ratio and overall return as metrics and examine the ability of DiffBid and IQL to control the overall CPC exceeding ratio while maximizing the overall return. During training, we set different thresholds of CPC as in Eq~(\ref{eq:constraint}). Then when testing, we make DiffBid generating trajectories under the expected CPC. In Figure~\ref{fig:ppc_constraint}, we show the exceeding ratio and overall return under different CPC constraints and training settings. From the figure, we find that DiffBid has the ability to control diverse levels of exceeding ratio while maintaining an intact return, surpassing IQL by a significant margin. Consequently, DiffBid holds a distinct advantage in effectively addressing MCB problems. We also study the performance under different advertiser feedbacks. During training we split the trajectories through thresholds of Eq.~(\ref{eq:constraint}) into high and low levels, and learn the conditional distribution under different levels. During generation, we adjust the condition and generate corresponding samples and summarize the metrics. The results for the statistic distribution of metrics for low level, high level and the original trajectories are shown in Figure~\ref{fig:feedback}. We find that the trajectory obtained from deploying DiffBid is well controlled by the condition.

\subsubsection{\textbf{Impact of Diffusion Steps}} We also study the overall performance under different diffusion steps, which is an important factor in influencing the efficiency and performance. The overall impact of diffusion steps with respect to different budgets is illustrated in Figure~\ref{fig:steps}. From the figure, we have the following discoveries. First of all, we observe that diffusion steps have a larger impact on advertisers with small budgets (1500 yuan). Secondly, larger budgets are not sensitive to the diffusion steps, where we can get the best result in most situations within 30 diffusion steps.
\begin{figure}[t]
    \centering
    \subfigure[Impact of Diffusion Steps]{\includegraphics[width=.46\linewidth]{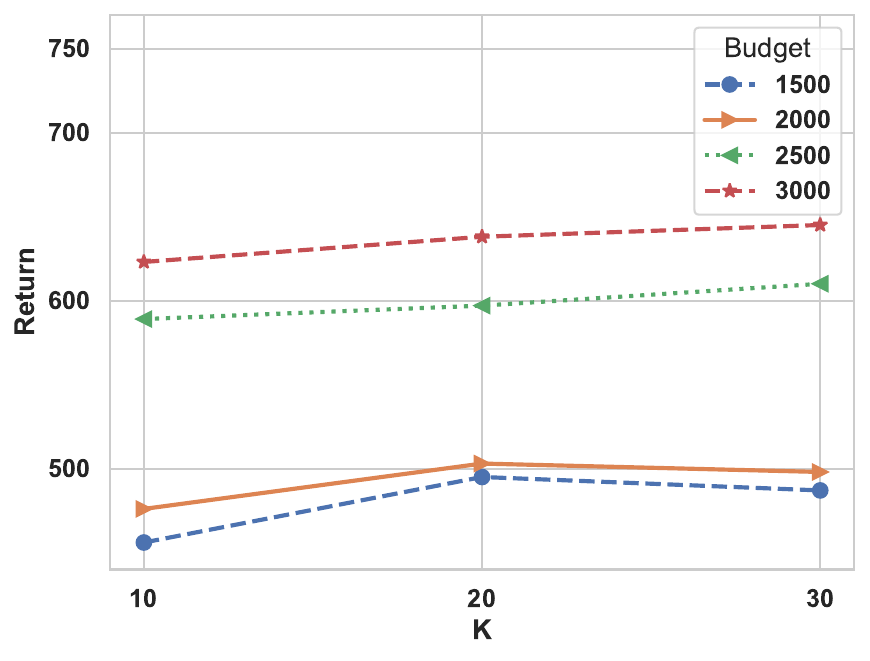}}
    \subfigure[Stability]{\includegraphics[width=.48\linewidth]{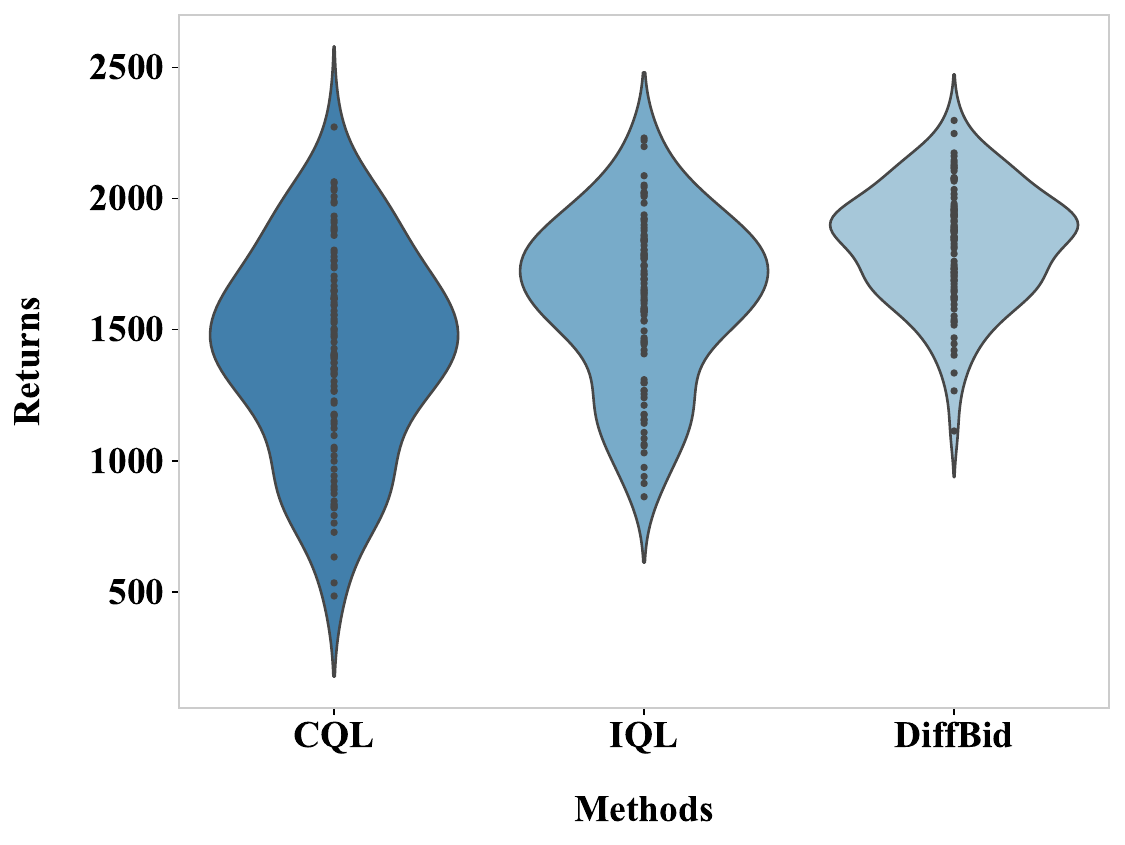}}
    \caption{In-depth Analysis.}
    \label{fig:steps}
\end{figure}
\subsubsection{\textbf{Stability}} 

In this study, we randomly initialized the parameters of three models - CQL, IQL, and DiffBid - and conducted thirty training trials for each to examine stability in performance. As depicted in Figure~\ref{fig:stable}, the RL-based models, CQL and IQL, showed a tendency towards instability under varying random seeds. Notably, IQL demonstrated slightly better performance than CQL, which may be attributed to its design optimized for conservative regularization. Contrasting with these, the generative model DiffBid exhibited remarkable stability, with significantly fewer instances of failure compared to its RL counterparts.

\subsection{Online A/B Test}

To further substantiate the effectiveness of DiffBid, we have deployed it on Alibaba advertising platform for comparison against the baseline
IQL~\cite{kostrikov2021offline} method, which performs best among various auto-bidding methods.
\begin{table}[t]
\Large
    \centering
    \caption{Online A/B Test Result.}
    \resizebox{.95\linewidth}{!}{
    \begin{tabular}{c|cccccc}
    \toprule
        Metrics  & \#Plan & Budget & Cost & Buycnt & GMV & ROI \\
    \midrule
       Baseline  & 2068 & 886744 & 834426.104 & 23584.6836 & 1853823 & 2.221 \\
       DiffBid   & 2068 & 886744 & 829992.384 & 24078.6883 &  1905954 & 2.296 \\
    \textit{compare} & - & - & -0.53\% & +2.09\% & +2.81\% & +3.36\% \\
    \bottomrule
    \end{tabular}}
    \label{tab:online}
\end{table}
The online A/B test is conducted from February 01, 2024, to February 08, 2024. The results are shown in Table~\ref{tab:online}. It shows that DiffBid can significantly improve the Buycnt by 2.09\%, the GMV by 2.81\%, the ROI by 3.36\%, showing its effectiveness in optimizing the overall performance. For efficiency, DiffBid takes 0.2s per request with GPU acceleration while the baseline is 0.07s, which means latency can be well guaranteed.
 
\section{Related Works}

\noindent\textbf{Offline-Reinforcement Learning.} Offline reinforcement learning is a research direction that has gained significant attention in recent years. The primary goal of offline RL is to learn effective policies from a fixed dataset without additional online interaction with the environment. This approach is particularly beneficial when online interaction is costly, risky, or otherwise not feasible.Notable works include Conservative Q-learning~(CQL) by Kumar et al. \cite{kumar2020conservative}, and Batch-Constrained deep Q-learning~(BCQ) by Fujimoto et al. \cite{fujimoto2019off}. Both algorithms aim to tackle overestimation bias which tends to occur in offline RL settings. Kostrikov et al.~\cite{kostrikov2021offline} propose an implicit q-learning approach to address the training instability for CQL. Chen et al.~\cite{chen2021decision} propose to use transformers for offline RL to increase the model capability. Hansen-Estruch et al.~\cite{hansen2023idql} proposes a diffusion-based approach with implicit Q-learning for offline RL.

\noindent\textbf{Diffusion Models.} They recently have shown the capability of high-quality generation~\cite{croitoru2023diffusion}, unconditional generation~\cite{austin2021structured} and conditional generation~\cite{chao2022denoising,huang2022fastdiff}. It has shown promising performance in decision-making. Hansen-Estruch et al.~\cite{hansen2023idql} proposes a diffusion-based approach with implicit q-learning for offline RL. Wang et al.~\cite{wang2022diffusion} propose a expressive policy though diffusion modeling. Chen et al.~\cite{chen2022offline} propose to use diffusion models for behavior modeling. Hu et al.~\cite{hu2023instructed} introduce temporal conditions for trajectory generation. Despite these preliminary explorations, no work has been payed for diffusion based auto-bidding which requires the model to adapt to the random advertising environment. Li et al.~\cite{li2023diga} utilize diffusion model in anti-money laundering.

\noindent\textbf{Auto-bidding.} Auto-bidding systems are widely used in programmatic advertising, where they are employed to automatically place bids on ad spaces. The main focus of such systems is to optimise a given key performance indicator (KPI), such as the number of clicks or conversions, while maintaining a certain budget~\cite{wang2017display}. Cai et al.~\cite{cai2017real} proposed an RL-based approach to the problem of auto-bidding for display advertising. They designed a bidding environment and applied a deep RL algorithm to learn the optimal bidding strategy. He et al.~\cite{he2021unified} propose a unified solution with RL to enable multiple constraints for auto-bidding.  Jin et al. extend the RL to enable multi-agent competition~\cite{jin2018real}. Zhang et al. \cite{mou2022sustainable} also adopted the RL framework for auto-bidding and showed that their approach can outperform traditional bidding strategies. Wen et al.~\cite{wen2022cooperative} propose a multi-agent-based approach for auto-bidding, which enables the modeling of multiple auto-bidding agents at the same time to include more information and also has been deployed online.
\section{Conclusion}

In this paper, we design a new paradigm for auto-bidding through the lens of generative modeling. To achieve this goal, we propose a decision-denoising diffusion approach to generate conditional bidding trajectories and at the same time control the generated samples under certain constraints. This new generative modeling approach enables integrating different kinds of industrial metrics, which is the first unified model for bidding. Extensive experiments on real-world simulation environments demonstrate the effectiveness of the newly proposed approach.  In the future, we will consider developing new methods to accelerate the generation process and new methods to ensure the robustness of DiffBid.  




\balance
\bibliographystyle{ACM-Reference-Format}
\bibliography{main}


\begin{thebibliography}{50}


\ifx \showCODEN    \undefined \def \showCODEN     #1{\unskip}     \fi
\ifx \showDOI      \undefined \def \showDOI       #1{#1}\fi
\ifx \showISBNx    \undefined \def \showISBNx     #1{\unskip}     \fi
\ifx \showISBNxiii \undefined \def \showISBNxiii  #1{\unskip}     \fi
\ifx \showISSN     \undefined \def \showISSN      #1{\unskip}     \fi
\ifx \showLCCN     \undefined \def \showLCCN      #1{\unskip}     \fi
\ifx \shownote     \undefined \def \shownote      #1{#1}          \fi
\ifx \showarticletitle \undefined \def \showarticletitle #1{#1}   \fi
\ifx \showURL      \undefined \def \showURL       {\relax}        \fi
\providecommand\bibfield[2]{#2}
\providecommand\bibinfo[2]{#2}
\providecommand\natexlab[1]{#1}
\providecommand\showeprint[2][]{arXiv:#2}

\bibitem[Agrawal et~al\mbox{.}(2016)]%
        {agrawal2016learning}
\bibfield{author}{\bibinfo{person}{Pulkit Agrawal}, \bibinfo{person}{Ashvin~V Nair}, \bibinfo{person}{Pieter Abbeel}, \bibinfo{person}{Jitendra Malik}, {and} \bibinfo{person}{Sergey Levine}.} \bibinfo{year}{2016}\natexlab{}.
\newblock \showarticletitle{Learning to poke by poking: Experiential learning of intuitive physics}.
\newblock \bibinfo{journal}{\emph{Advances in neural information processing systems}}  \bibinfo{volume}{29} (\bibinfo{year}{2016}).
\newblock


\bibitem[Ajay et~al\mbox{.}(2022)]%
        {ajay2022conditional}
\bibfield{author}{\bibinfo{person}{Anurag Ajay}, \bibinfo{person}{Yilun Du}, \bibinfo{person}{Abhi Gupta}, \bibinfo{person}{Joshua~B Tenenbaum}, \bibinfo{person}{Tommi~S Jaakkola}, {and} \bibinfo{person}{Pulkit Agrawal}.} \bibinfo{year}{2022}\natexlab{}.
\newblock \showarticletitle{Is Conditional Generative Modeling all you need for Decision Making?}. In \bibinfo{booktitle}{\emph{The Eleventh International Conference on Learning Representations}}.
\newblock


\bibitem[Austin et~al\mbox{.}(2021)]%
        {austin2021structured}
\bibfield{author}{\bibinfo{person}{Jacob Austin}, \bibinfo{person}{Daniel~D Johnson}, \bibinfo{person}{Jonathan Ho}, \bibinfo{person}{Daniel Tarlow}, {and} \bibinfo{person}{Rianne Van Den~Berg}.} \bibinfo{year}{2021}\natexlab{}.
\newblock \showarticletitle{Structured denoising diffusion models in discrete state-spaces}.
\newblock \bibinfo{journal}{\emph{Advances in Neural Information Processing Systems}}  \bibinfo{volume}{34} (\bibinfo{year}{2021}), \bibinfo{pages}{17981--17993}.
\newblock


\bibitem[Balseiro et~al\mbox{.}(2021a)]%
        {balseiro2021robust}
\bibfield{author}{\bibinfo{person}{Santiago Balseiro}, \bibinfo{person}{Yuan Deng}, \bibinfo{person}{Jieming Mao}, \bibinfo{person}{Vahab Mirrokni}, {and} \bibinfo{person}{Song Zuo}.} \bibinfo{year}{2021}\natexlab{a}.
\newblock \showarticletitle{Robust auction design in the auto-bidding world}.
\newblock \bibinfo{journal}{\emph{Advances in Neural Information Processing Systems}}  \bibinfo{volume}{34} (\bibinfo{year}{2021}), \bibinfo{pages}{17777--17788}.
\newblock


\bibitem[Balseiro et~al\mbox{.}(2021b)]%
        {balseiro2021landscape}
\bibfield{author}{\bibinfo{person}{Santiago~R Balseiro}, \bibinfo{person}{Yuan Deng}, \bibinfo{person}{Jieming Mao}, \bibinfo{person}{Vahab~S Mirrokni}, {and} \bibinfo{person}{Song Zuo}.} \bibinfo{year}{2021}\natexlab{b}.
\newblock \showarticletitle{The landscape of auto-bidding auctions: Value versus utility maximization}. In \bibinfo{booktitle}{\emph{Proceedings of the 22nd ACM Conference on Economics and Computation}}. \bibinfo{pages}{132--133}.
\newblock


\bibitem[Cai et~al\mbox{.}(2017)]%
        {cai2017real}
\bibfield{author}{\bibinfo{person}{Han Cai}, \bibinfo{person}{Kan Ren}, \bibinfo{person}{Weinan Zhang}, \bibinfo{person}{Kleanthis Malialis}, \bibinfo{person}{Jun Wang}, \bibinfo{person}{Yong Yu}, {and} \bibinfo{person}{Defeng Guo}.} \bibinfo{year}{2017}\natexlab{}.
\newblock \showarticletitle{Real-time bidding by reinforcement learning in display advertising}. In \bibinfo{booktitle}{\emph{Proceedings of the tenth ACM international conference on web search and data mining}}. \bibinfo{pages}{661--670}.
\newblock


\bibitem[Chao et~al\mbox{.}(2022)]%
        {chao2022denoising}
\bibfield{author}{\bibinfo{person}{Chen-Hao Chao}, \bibinfo{person}{Wei-Fang Sun}, \bibinfo{person}{Bo-Wun Cheng}, \bibinfo{person}{Yi-Chen Lo}, \bibinfo{person}{Chia-Che Chang}, \bibinfo{person}{Yu-Lun Liu}, \bibinfo{person}{Yu-Lin Chang}, \bibinfo{person}{Chia-Ping Chen}, {and} \bibinfo{person}{Chun-Yi Lee}.} \bibinfo{year}{2022}\natexlab{}.
\newblock \showarticletitle{Denoising likelihood score matching for conditional score-based data generation}.
\newblock \bibinfo{journal}{\emph{arXiv preprint arXiv:2203.14206}} (\bibinfo{year}{2022}).
\newblock


\bibitem[Chen et~al\mbox{.}(2022)]%
        {chen2022offline}
\bibfield{author}{\bibinfo{person}{Huayu Chen}, \bibinfo{person}{Cheng Lu}, \bibinfo{person}{Chengyang Ying}, \bibinfo{person}{Hang Su}, {and} \bibinfo{person}{Jun Zhu}.} \bibinfo{year}{2022}\natexlab{}.
\newblock \showarticletitle{Offline reinforcement learning via high-fidelity generative behavior modeling}.
\newblock \bibinfo{journal}{\emph{arXiv preprint arXiv:2209.14548}} (\bibinfo{year}{2022}).
\newblock


\bibitem[Chen et~al\mbox{.}(2021)]%
        {chen2021decision}
\bibfield{author}{\bibinfo{person}{Lili Chen}, \bibinfo{person}{Kevin Lu}, \bibinfo{person}{Aravind Rajeswaran}, \bibinfo{person}{Kimin Lee}, \bibinfo{person}{Aditya Grover}, \bibinfo{person}{Misha Laskin}, \bibinfo{person}{Pieter Abbeel}, \bibinfo{person}{Aravind Srinivas}, {and} \bibinfo{person}{Igor Mordatch}.} \bibinfo{year}{2021}\natexlab{}.
\newblock \showarticletitle{Decision transformer: Reinforcement learning via sequence modeling}.
\newblock \bibinfo{journal}{\emph{Advances in neural information processing systems}}  \bibinfo{volume}{34} (\bibinfo{year}{2021}), \bibinfo{pages}{15084--15097}.
\newblock


\bibitem[Chiesi et~al\mbox{.}(1979)]%
        {chiesi1979acquisition}
\bibfield{author}{\bibinfo{person}{Harry~L Chiesi}, \bibinfo{person}{George~J Spilich}, {and} \bibinfo{person}{James~F Voss}.} \bibinfo{year}{1979}\natexlab{}.
\newblock \showarticletitle{Acquisition of domain-related information in relation to high and low domain knowledge}.
\newblock \bibinfo{journal}{\emph{Journal of verbal learning and verbal behavior}} \bibinfo{volume}{18}, \bibinfo{number}{3} (\bibinfo{year}{1979}), \bibinfo{pages}{257--273}.
\newblock


\bibitem[Croitoru et~al\mbox{.}(2023)]%
        {croitoru2023diffusion}
\bibfield{author}{\bibinfo{person}{Florinel-Alin Croitoru}, \bibinfo{person}{Vlad Hondru}, \bibinfo{person}{Radu~Tudor Ionescu}, {and} \bibinfo{person}{Mubarak Shah}.} \bibinfo{year}{2023}\natexlab{}.
\newblock \showarticletitle{Diffusion models in vision: A survey}.
\newblock \bibinfo{journal}{\emph{IEEE Transactions on Pattern Analysis and Machine Intelligence}} (\bibinfo{year}{2023}).
\newblock


\bibitem[Deng et~al\mbox{.}(2021)]%
        {deng2021towards}
\bibfield{author}{\bibinfo{person}{Yuan Deng}, \bibinfo{person}{Jieming Mao}, \bibinfo{person}{Vahab Mirrokni}, {and} \bibinfo{person}{Song Zuo}.} \bibinfo{year}{2021}\natexlab{}.
\newblock \showarticletitle{Towards efficient auctions in an auto-bidding world}. In \bibinfo{booktitle}{\emph{Proceedings of the Web Conference 2021}}. \bibinfo{pages}{3965--3973}.
\newblock


\bibitem[Evans(2009)]%
        {evans2009online}
\bibfield{author}{\bibinfo{person}{David~S Evans}.} \bibinfo{year}{2009}\natexlab{}.
\newblock \showarticletitle{The online advertising industry: Economics, evolution, and privacy}.
\newblock \bibinfo{journal}{\emph{Journal of economic perspectives}} \bibinfo{volume}{23}, \bibinfo{number}{3} (\bibinfo{year}{2009}), \bibinfo{pages}{37--60}.
\newblock


\bibitem[Fujimoto et~al\mbox{.}(2019)]%
        {fujimoto2019off}
\bibfield{author}{\bibinfo{person}{Scott Fujimoto}, \bibinfo{person}{David Meger}, {and} \bibinfo{person}{Doina Precup}.} \bibinfo{year}{2019}\natexlab{}.
\newblock \showarticletitle{Off-policy deep reinforcement learning without exploration}. In \bibinfo{booktitle}{\emph{International conference on machine learning}}. PMLR, \bibinfo{pages}{2052--2062}.
\newblock


\bibitem[Fujimoto et~al\mbox{.}(2022)]%
        {pmlr-v162-fujimoto22a}
\bibfield{author}{\bibinfo{person}{Scott Fujimoto}, \bibinfo{person}{David Meger}, \bibinfo{person}{Doina Precup}, \bibinfo{person}{Ofir Nachum}, {and} \bibinfo{person}{Shixiang~Shane Gu}.} \bibinfo{year}{2022}\natexlab{}.
\newblock \showarticletitle{Why Should I Trust You, Bellman? {T}he {B}ellman Error is a Poor Replacement for Value Error}. In \bibinfo{booktitle}{\emph{Proceedings of the 39th International Conference on Machine Learning}} \emph{(\bibinfo{series}{Proceedings of Machine Learning Research}, Vol.~\bibinfo{volume}{162})}, \bibfield{editor}{\bibinfo{person}{Kamalika Chaudhuri}, \bibinfo{person}{Stefanie Jegelka}, \bibinfo{person}{Le~Song}, \bibinfo{person}{Csaba Szepesvari}, \bibinfo{person}{Gang Niu}, {and} \bibinfo{person}{Sivan Sabato}} (Eds.). \bibinfo{publisher}{PMLR}, \bibinfo{pages}{6918--6943}.
\newblock
\urldef\tempurl%
\url{https://proceedings.mlr.press/v162/fujimoto22a.html}
\showURL{%
\tempurl}


\bibitem[Gaon and Brafman(2020)]%
        {gaon2020reinforcement}
\bibfield{author}{\bibinfo{person}{Maor Gaon} {and} \bibinfo{person}{Ronen Brafman}.} \bibinfo{year}{2020}\natexlab{}.
\newblock \showarticletitle{Reinforcement learning with non-markovian rewards}. In \bibinfo{booktitle}{\emph{Proceedings of the AAAI conference on artificial intelligence}}, Vol.~\bibinfo{volume}{34}. \bibinfo{pages}{3980--3987}.
\newblock


\bibitem[Guo et~al\mbox{.}(2022)]%
        {10.1145/3488560.3498524}
\bibfield{author}{\bibinfo{person}{Jiayan Guo}, \bibinfo{person}{Yaming Yang}, \bibinfo{person}{Xiangchen Song}, \bibinfo{person}{Yuan Zhang}, \bibinfo{person}{Yujing Wang}, \bibinfo{person}{Jing Bai}, {and} \bibinfo{person}{Yan Zhang}.} \bibinfo{year}{2022}\natexlab{}.
\newblock \showarticletitle{Learning Multi-granularity Consecutive User Intent Unit for Session-based Recommendation}. In \bibinfo{booktitle}{\emph{Proceedings of the Fifteenth ACM International Conference on Web Search and Data Mining}} (Virtual Event, AZ, USA) \emph{(\bibinfo{series}{WSDM '22})}. \bibinfo{publisher}{Association for Computing Machinery}, \bibinfo{address}{New York, NY, USA}, \bibinfo{pages}{343–352}.
\newblock
\showISBNx{9781450391320}
\urldef\tempurl%
\url{https://doi.org/10.1145/3488560.3498524}
\showDOI{\tempurl}


\bibitem[Ha(2008)]%
        {ha2008online}
\bibfield{author}{\bibinfo{person}{Louisa Ha}.} \bibinfo{year}{2008}\natexlab{}.
\newblock \showarticletitle{Online advertising research in advertising journals: A review}.
\newblock \bibinfo{journal}{\emph{Journal of Current Issues \& Research in Advertising}} \bibinfo{volume}{30}, \bibinfo{number}{1} (\bibinfo{year}{2008}), \bibinfo{pages}{31--48}.
\newblock


\bibitem[Hansen-Estruch et~al\mbox{.}(2023)]%
        {hansen2023idql}
\bibfield{author}{\bibinfo{person}{Philippe Hansen-Estruch}, \bibinfo{person}{Ilya Kostrikov}, \bibinfo{person}{Michael Janner}, \bibinfo{person}{Jakub~Grudzien Kuba}, {and} \bibinfo{person}{Sergey Levine}.} \bibinfo{year}{2023}\natexlab{}.
\newblock \showarticletitle{Idql: Implicit q-learning as an actor-critic method with diffusion policies}.
\newblock \bibinfo{journal}{\emph{arXiv preprint arXiv:2304.10573}} (\bibinfo{year}{2023}).
\newblock


\bibitem[Hao et~al\mbox{.}(2020)]%
        {DBLP:conf/icml/HaoPMWJHCBXXZYL20}
\bibfield{author}{\bibinfo{person}{Xiaotian Hao}, \bibinfo{person}{Zhaoqing Peng}, \bibinfo{person}{Yi Ma}, \bibinfo{person}{Guan Wang}, \bibinfo{person}{Junqi Jin}, \bibinfo{person}{Jianye Hao}, \bibinfo{person}{Shan Chen}, \bibinfo{person}{Rongquan Bai}, \bibinfo{person}{Mingzhou Xie}, \bibinfo{person}{Miao Xu}, \bibinfo{person}{Zhenzhe Zheng}, \bibinfo{person}{Chuan Yu}, \bibinfo{person}{Han Li}, \bibinfo{person}{Jian Xu}, {and} \bibinfo{person}{Kun Gai}.} \bibinfo{year}{2020}\natexlab{}.
\newblock \showarticletitle{Dynamic Knapsack Optimization Towards Efficient Multi-Channel Sequential Advertising}. In \bibinfo{booktitle}{\emph{Proceedings of the 37th International Conference on Machine Learning, {ICML} 2020, 13-18 July 2020, Virtual Event}} \emph{(\bibinfo{series}{Proceedings of Machine Learning Research}, Vol.~\bibinfo{volume}{119})}. \bibinfo{publisher}{{PMLR}}, \bibinfo{pages}{4060--4070}.
\newblock
\urldef\tempurl%
\url{http://proceedings.mlr.press/v119/hao20b.html}
\showURL{%
\tempurl}


\bibitem[He et~al\mbox{.}(2021)]%
        {he2021unified}
\bibfield{author}{\bibinfo{person}{Yue He}, \bibinfo{person}{Xiujun Chen}, \bibinfo{person}{Di Wu}, \bibinfo{person}{Junwei Pan}, \bibinfo{person}{Qing Tan}, \bibinfo{person}{Chuan Yu}, \bibinfo{person}{Jian Xu}, {and} \bibinfo{person}{Xiaoqiang Zhu}.} \bibinfo{year}{2021}\natexlab{}.
\newblock \showarticletitle{A unified solution to constrained bidding in online display advertising}. In \bibinfo{booktitle}{\emph{Proceedings of the 27th ACM SIGKDD Conference on Knowledge Discovery \& Data Mining}}. \bibinfo{pages}{2993--3001}.
\newblock


\bibitem[Ho et~al\mbox{.}(2020)]%
        {ho2020denoising}
\bibfield{author}{\bibinfo{person}{Jonathan Ho}, \bibinfo{person}{Ajay Jain}, {and} \bibinfo{person}{Pieter Abbeel}.} \bibinfo{year}{2020}\natexlab{}.
\newblock \showarticletitle{Denoising diffusion probabilistic models}.
\newblock \bibinfo{journal}{\emph{Advances in neural information processing systems}}  \bibinfo{volume}{33} (\bibinfo{year}{2020}), \bibinfo{pages}{6840--6851}.
\newblock


\bibitem[Ho and Salimans(2021)]%
        {ho2021classifier}
\bibfield{author}{\bibinfo{person}{Jonathan Ho} {and} \bibinfo{person}{Tim Salimans}.} \bibinfo{year}{2021}\natexlab{}.
\newblock \showarticletitle{Classifier-Free Diffusion Guidance}. In \bibinfo{booktitle}{\emph{NeurIPS 2021 Workshop on Deep Generative Models and Downstream Applications}}.
\newblock


\bibitem[Hu et~al\mbox{.}(2023)]%
        {hu2023instructed}
\bibfield{author}{\bibinfo{person}{Jifeng Hu}, \bibinfo{person}{Yanchao Sun}, \bibinfo{person}{Sili Huang}, \bibinfo{person}{SiYuan Guo}, \bibinfo{person}{Hechang Chen}, \bibinfo{person}{Li Shen}, \bibinfo{person}{Lichao Sun}, \bibinfo{person}{Yi Chang}, {and} \bibinfo{person}{Dacheng Tao}.} \bibinfo{year}{2023}\natexlab{}.
\newblock \showarticletitle{Instructed Diffuser with Temporal Condition Guidance for Offline Reinforcement Learning}.
\newblock \bibinfo{journal}{\emph{arXiv preprint arXiv:2306.04875}} (\bibinfo{year}{2023}).
\newblock


\bibitem[Huang et~al\mbox{.}(2022)]%
        {huang2022fastdiff}
\bibfield{author}{\bibinfo{person}{R Huang}, \bibinfo{person}{MWY Lam}, \bibinfo{person}{J Wang}, \bibinfo{person}{D Su}, \bibinfo{person}{D Yu}, \bibinfo{person}{Y Ren}, {and} \bibinfo{person}{Z Zhao}.} \bibinfo{year}{2022}\natexlab{}.
\newblock \showarticletitle{FastDiff: A Fast Conditional Diffusion Model for High-Quality Speech Synthesis}. In \bibinfo{booktitle}{\emph{IJCAI International Joint Conference on Artificial Intelligence}}. IJCAI: International Joint Conferences on Artificial Intelligence Organization, \bibinfo{pages}{4157--4163}.
\newblock


\bibitem[Jaynes(1957)]%
        {jaynes1957information}
\bibfield{author}{\bibinfo{person}{Edwin~T Jaynes}.} \bibinfo{year}{1957}\natexlab{}.
\newblock \showarticletitle{Information theory and statistical mechanics}.
\newblock \bibinfo{journal}{\emph{Physical review}} \bibinfo{volume}{106}, \bibinfo{number}{4} (\bibinfo{year}{1957}), \bibinfo{pages}{620}.
\newblock


\bibitem[Jin et~al\mbox{.}(2018)]%
        {jin2018real}
\bibfield{author}{\bibinfo{person}{Junqi Jin}, \bibinfo{person}{Chengru Song}, \bibinfo{person}{Han Li}, \bibinfo{person}{Kun Gai}, \bibinfo{person}{Jun Wang}, {and} \bibinfo{person}{Weinan Zhang}.} \bibinfo{year}{2018}\natexlab{}.
\newblock \showarticletitle{Real-time bidding with multi-agent reinforcement learning in display advertising}. In \bibinfo{booktitle}{\emph{Proceedings of the 27th ACM international conference on information and knowledge management}}. \bibinfo{pages}{2193--2201}.
\newblock


\bibitem[Kingma et~al\mbox{.}(2015)]%
        {kingma2015variational}
\bibfield{author}{\bibinfo{person}{Durk~P Kingma}, \bibinfo{person}{Tim Salimans}, {and} \bibinfo{person}{Max Welling}.} \bibinfo{year}{2015}\natexlab{}.
\newblock \showarticletitle{Variational dropout and the local reparameterization trick}.
\newblock \bibinfo{journal}{\emph{Advances in neural information processing systems}}  \bibinfo{volume}{28} (\bibinfo{year}{2015}).
\newblock


\bibitem[Kostrikov et~al\mbox{.}(2021)]%
        {kostrikov2021offline}
\bibfield{author}{\bibinfo{person}{Ilya Kostrikov}, \bibinfo{person}{Ashvin Nair}, {and} \bibinfo{person}{Sergey Levine}.} \bibinfo{year}{2021}\natexlab{}.
\newblock \showarticletitle{Offline Reinforcement Learning with Implicit Q-Learning}. In \bibinfo{booktitle}{\emph{International Conference on Learning Representations}}.
\newblock


\bibitem[Kumar et~al\mbox{.}(2020)]%
        {kumar2020conservative}
\bibfield{author}{\bibinfo{person}{Aviral Kumar}, \bibinfo{person}{Aurick Zhou}, \bibinfo{person}{George Tucker}, {and} \bibinfo{person}{Sergey Levine}.} \bibinfo{year}{2020}\natexlab{}.
\newblock \showarticletitle{Conservative q-learning for offline reinforcement learning}.
\newblock \bibinfo{journal}{\emph{Advances in Neural Information Processing Systems}}  \bibinfo{volume}{33} (\bibinfo{year}{2020}), \bibinfo{pages}{1179--1191}.
\newblock


\bibitem[Li and Tang(2022)]%
        {li2022auto}
\bibfield{author}{\bibinfo{person}{Juncheng Li} {and} \bibinfo{person}{Pingzhong Tang}.} \bibinfo{year}{2022}\natexlab{}.
\newblock \showarticletitle{Auto-bidding Equilibrium in ROI-Constrained Online Advertising Markets}.
\newblock \bibinfo{journal}{\emph{arXiv preprint arXiv:2210.06107}} (\bibinfo{year}{2022}).
\newblock


\bibitem[Li et~al\mbox{.}(2023)]%
        {li2023diga}
\bibfield{author}{\bibinfo{person}{Xujia Li}, \bibinfo{person}{Yuan Li}, \bibinfo{person}{Xueying Mo}, \bibinfo{person}{Hebing Xiao}, \bibinfo{person}{Yanyan Shen}, {and} \bibinfo{person}{Lei Chen}.} \bibinfo{year}{2023}\natexlab{}.
\newblock \showarticletitle{Diga: Guided diffusion model for graph recovery in anti-money laundering}. In \bibinfo{booktitle}{\emph{Proceedings of the 29th ACM SIGKDD Conference on Knowledge Discovery and Data Mining}}. \bibinfo{pages}{4404--4413}.
\newblock


\bibitem[Majeed and Hutter(2018)]%
        {majeed2018q}
\bibfield{author}{\bibinfo{person}{Sultan~Javed Majeed} {and} \bibinfo{person}{Marcus Hutter}.} \bibinfo{year}{2018}\natexlab{}.
\newblock \showarticletitle{On Q-learning Convergence for Non-Markov Decision Processes.}. In \bibinfo{booktitle}{\emph{IJCAI}}, Vol.~\bibinfo{volume}{18}. \bibinfo{pages}{2546--2552}.
\newblock


\bibitem[Misra(2019)]%
        {misra2019mish}
\bibfield{author}{\bibinfo{person}{Diganta Misra}.} \bibinfo{year}{2019}\natexlab{}.
\newblock \showarticletitle{Mish: A self regularized non-monotonic activation function}.
\newblock \bibinfo{journal}{\emph{arXiv preprint arXiv:1908.08681}} (\bibinfo{year}{2019}).
\newblock


\bibitem[Mou et~al\mbox{.}(2022)]%
        {mou2022sustainable}
\bibfield{author}{\bibinfo{person}{Zhiyu Mou}, \bibinfo{person}{Yusen Huo}, \bibinfo{person}{Rongquan Bai}, \bibinfo{person}{Mingzhou Xie}, \bibinfo{person}{Chuan Yu}, \bibinfo{person}{Jian Xu}, {and} \bibinfo{person}{Bo Zheng}.} \bibinfo{year}{2022}\natexlab{}.
\newblock \showarticletitle{Sustainable Online Reinforcement Learning for Auto-bidding}.
\newblock \bibinfo{journal}{\emph{Advances in Neural Information Processing Systems}}  \bibinfo{volume}{35} (\bibinfo{year}{2022}), \bibinfo{pages}{2651--2663}.
\newblock


\bibitem[Mutti et~al\mbox{.}(2022)]%
        {mutti2022importance}
\bibfield{author}{\bibinfo{person}{Mirco Mutti}, \bibinfo{person}{Riccardo De~Santi}, {and} \bibinfo{person}{Marcello Restelli}.} \bibinfo{year}{2022}\natexlab{}.
\newblock \showarticletitle{The importance of non-markovianity in maximum state entropy exploration}. In \bibinfo{booktitle}{\emph{International Conference on Machine Learning}}. PMLR, \bibinfo{pages}{16223--16239}.
\newblock


\bibitem[Nguyen-Tuong et~al\mbox{.}(2008)]%
        {nguyen2008learning}
\bibfield{author}{\bibinfo{person}{Duy Nguyen-Tuong}, \bibinfo{person}{Jan Peters}, \bibinfo{person}{Matthias Seeger}, {and} \bibinfo{person}{Bernhard Sch{\"o}lkopf}.} \bibinfo{year}{2008}\natexlab{}.
\newblock \showarticletitle{Learning inverse dynamics: a comparison}. In \bibinfo{booktitle}{\emph{European symposium on artificial neural networks}}.
\newblock


\bibitem[Nichol and Dhariwal(2021)]%
        {nichol2021improved}
\bibfield{author}{\bibinfo{person}{Alexander~Quinn Nichol} {and} \bibinfo{person}{Prafulla Dhariwal}.} \bibinfo{year}{2021}\natexlab{}.
\newblock \showarticletitle{Improved denoising diffusion probabilistic models}. In \bibinfo{booktitle}{\emph{International Conference on Machine Learning}}. PMLR, \bibinfo{pages}{8162--8171}.
\newblock


\bibitem[Ou et~al\mbox{.}(2023)]%
        {ou2023deep}
\bibfield{author}{\bibinfo{person}{Weitong Ou}, \bibinfo{person}{Bo Chen}, \bibinfo{person}{Yingxuan Yang}, \bibinfo{person}{Xinyi Dai}, \bibinfo{person}{Weiwen Liu}, \bibinfo{person}{Weinan Zhang}, \bibinfo{person}{Ruiming Tang}, {and} \bibinfo{person}{Yong Yu}.} \bibinfo{year}{2023}\natexlab{}.
\newblock \showarticletitle{Deep landscape forecasting in multi-slot real-time bidding}. In \bibinfo{booktitle}{\emph{Proceedings of the 29th ACM SIGKDD Conference on Knowledge Discovery and Data Mining}}. \bibinfo{pages}{4685--4695}.
\newblock


\bibitem[Pathak et~al\mbox{.}(2018)]%
        {pathak2018zero}
\bibfield{author}{\bibinfo{person}{Deepak Pathak}, \bibinfo{person}{Parsa Mahmoudieh}, \bibinfo{person}{Guanghao Luo}, \bibinfo{person}{Pulkit Agrawal}, \bibinfo{person}{Dian Chen}, \bibinfo{person}{Yide Shentu}, \bibinfo{person}{Evan Shelhamer}, \bibinfo{person}{Jitendra Malik}, \bibinfo{person}{Alexei~A Efros}, {and} \bibinfo{person}{Trevor Darrell}.} \bibinfo{year}{2018}\natexlab{}.
\newblock \showarticletitle{Zero-shot visual imitation}. In \bibinfo{booktitle}{\emph{Proceedings of the IEEE conference on computer vision and pattern recognition workshops}}. \bibinfo{pages}{2050--2053}.
\newblock


\bibitem[Qin et~al\mbox{.}(2023)]%
        {qin2023learning}
\bibfield{author}{\bibinfo{person}{Aoyang Qin}, \bibinfo{person}{Feng Gao}, \bibinfo{person}{Qing Li}, \bibinfo{person}{Song-Chun Zhu}, {and} \bibinfo{person}{Sirui Xie}.} \bibinfo{year}{2023}\natexlab{}.
\newblock \showarticletitle{Learning non-Markovian Decision-Making from State-only Sequences}. In \bibinfo{booktitle}{\emph{Thirty-seventh Conference on Neural Information Processing Systems}}.
\newblock


\bibitem[Ronneberger et~al\mbox{.}(2015)]%
        {ronneberger2015u}
\bibfield{author}{\bibinfo{person}{Olaf Ronneberger}, \bibinfo{person}{Philipp Fischer}, {and} \bibinfo{person}{Thomas Brox}.} \bibinfo{year}{2015}\natexlab{}.
\newblock \showarticletitle{U-net: Convolutional networks for biomedical image segmentation}. In \bibinfo{booktitle}{\emph{Medical Image Computing and Computer-Assisted Intervention--MICCAI 2015: 18th International Conference, Munich, Germany, October 5-9, 2015, Proceedings, Part III 18}}. Springer, \bibinfo{pages}{234--241}.
\newblock


\bibitem[Vincent(2011)]%
        {vincent2011connection}
\bibfield{author}{\bibinfo{person}{Pascal Vincent}.} \bibinfo{year}{2011}\natexlab{}.
\newblock \showarticletitle{A connection between score matching and denoising autoencoders}.
\newblock \bibinfo{journal}{\emph{Neural computation}} \bibinfo{volume}{23}, \bibinfo{number}{7} (\bibinfo{year}{2011}), \bibinfo{pages}{1661--1674}.
\newblock


\bibitem[Wang et~al\mbox{.}(2017)]%
        {wang2017display}
\bibfield{author}{\bibinfo{person}{Jun Wang}, \bibinfo{person}{Weinan Zhang}, \bibinfo{person}{Shuai Yuan}, {et~al\mbox{.}}} \bibinfo{year}{2017}\natexlab{}.
\newblock \showarticletitle{Display advertising with real-time bidding (RTB) and behavioural targeting}.
\newblock \bibinfo{journal}{\emph{Foundations and Trends{\textregistered} in Information Retrieval}} \bibinfo{volume}{11}, \bibinfo{number}{4-5} (\bibinfo{year}{2017}), \bibinfo{pages}{297--435}.
\newblock


\bibitem[Wang et~al\mbox{.}(2022)]%
        {wang2022diffusion}
\bibfield{author}{\bibinfo{person}{Zhendong Wang}, \bibinfo{person}{Jonathan~J Hunt}, {and} \bibinfo{person}{Mingyuan Zhou}.} \bibinfo{year}{2022}\natexlab{}.
\newblock \showarticletitle{Diffusion Policies as an Expressive Policy Class for Offline Reinforcement Learning}. In \bibinfo{booktitle}{\emph{The Eleventh International Conference on Learning Representations}}.
\newblock


\bibitem[Wen et~al\mbox{.}(2022)]%
        {wen2022cooperative}
\bibfield{author}{\bibinfo{person}{Chao Wen}, \bibinfo{person}{Miao Xu}, \bibinfo{person}{Zhilin Zhang}, \bibinfo{person}{Zhenzhe Zheng}, \bibinfo{person}{Yuhui Wang}, \bibinfo{person}{Xiangyu Liu}, \bibinfo{person}{Yu Rong}, \bibinfo{person}{Dong Xie}, \bibinfo{person}{Xiaoyang Tan}, \bibinfo{person}{Chuan Yu}, {et~al\mbox{.}}} \bibinfo{year}{2022}\natexlab{}.
\newblock \showarticletitle{A cooperative-competitive multi-agent framework for auto-bidding in online advertising}. In \bibinfo{booktitle}{\emph{Proceedings of the Fifteenth ACM International Conference on Web Search and Data Mining}}. \bibinfo{pages}{1129--1139}.
\newblock


\bibitem[Wu and He(2018)]%
        {wu2018group}
\bibfield{author}{\bibinfo{person}{Yuxin Wu} {and} \bibinfo{person}{Kaiming He}.} \bibinfo{year}{2018}\natexlab{}.
\newblock \showarticletitle{Group normalization}. In \bibinfo{booktitle}{\emph{Proceedings of the European conference on computer vision (ECCV)}}. \bibinfo{pages}{3--19}.
\newblock


\bibitem[Zhang et~al\mbox{.}(2023b)]%
        {zhang2023personalized}
\bibfield{author}{\bibinfo{person}{Haoqi Zhang}, \bibinfo{person}{Lvyin Niu}, \bibinfo{person}{Zhenzhe Zheng}, \bibinfo{person}{Zhilin Zhang}, \bibinfo{person}{Shan Gu}, \bibinfo{person}{Fan Wu}, \bibinfo{person}{Chuan Yu}, \bibinfo{person}{Jian Xu}, \bibinfo{person}{Guihai Chen}, {and} \bibinfo{person}{Bo Zheng}.} \bibinfo{year}{2023}\natexlab{b}.
\newblock \showarticletitle{A Personalized Automated Bidding Framework for Fairness-aware Online Advertising}. In \bibinfo{booktitle}{\emph{Proceedings of the 29th ACM SIGKDD Conference on Knowledge Discovery and Data Mining}}. \bibinfo{pages}{5544--5553}.
\newblock


\bibitem[Zhang et~al\mbox{.}(2023a)]%
        {10.1145/3539597.3570445}
\bibfield{author}{\bibinfo{person}{Peiyan Zhang}, \bibinfo{person}{Jiayan Guo}, \bibinfo{person}{Chaozhuo Li}, \bibinfo{person}{Yueqi Xie}, \bibinfo{person}{Jae~Boum Kim}, \bibinfo{person}{Yan Zhang}, \bibinfo{person}{Xing Xie}, \bibinfo{person}{Haohan Wang}, {and} \bibinfo{person}{Sunghun Kim}.} \bibinfo{year}{2023}\natexlab{a}.
\newblock \showarticletitle{Efficiently Leveraging Multi-level User Intent for Session-based Recommendation via Atten-Mixer Network}. In \bibinfo{booktitle}{\emph{Proceedings of the Sixteenth ACM International Conference on Web Search and Data Mining}} (, Singapore, Singapore,) \emph{(\bibinfo{series}{WSDM '23})}. \bibinfo{publisher}{Association for Computing Machinery}, \bibinfo{address}{New York, NY, USA}, \bibinfo{pages}{168–176}.
\newblock
\showISBNx{9781450394079}
\urldef\tempurl%
\url{https://doi.org/10.1145/3539597.3570445}
\showDOI{\tempurl}


\bibitem[Ziebart(2010)]%
        {10.5555/2049078}
\bibfield{author}{\bibinfo{person}{Brian~D. Ziebart}.} \bibinfo{year}{2010}\natexlab{}.
\newblock \emph{\bibinfo{title}{Modeling purposeful adaptive behavior with the principle of maximum causal entropy}}.
\newblock \bibinfo{thesistype}{Ph.\,D. Dissertation}. \bibinfo{address}{USA}.
\newblock Advisor(s) Bagnell, J. Andrew.
\newblock
\showISBNx{9781124414218}
\newblock
\shownote{AAI3438449}.


\end{thebibliography}


\appendix

\section{Appendix}

\subsection{Notations}

\begin{table}[h]
    \centering
    \caption{Definition of Notations.}
    \begin{tabular}{c|c}
    \toprule
        Symbol & Definition \\
    \midrule
       $\tau$  &  The trajectory index of a serving policy.  \\
       $\bs{x}(\tau)_k $  &  Sequence of states of trajectory $\tau$ in diffusion step $k$. \\
       $\bs{y}(\tau)$ & Properties or conditions for $\tau$. \\ 
       $R$ & Return of a trajectory. \\
       $E$ & Binary indicator variable. \\
       $B$ & The budget of the advertiser. \\
       $C_i$ & The $i$'s constraint. \\
       $\bs{o}_i$ & Whether the advertiser wins impression $i$. \\
       $\bs{v}_i$ & The true value of the impression $i$. \\
       $\bs{b}_i^*$ & The optimal bidding price for the impression $i$. \\
       $\bs{s}_t$ & The state at time period $t$. \\
       $\bs{\hat{a}}_t$ & Predicted bidding parameters at time period $t$. \\
       $\lambda_i$ & The bidding parameters. \\ 
       $\bs{\epsilon}_\theta$ & The denoising model that predict the noise. \\
       $f_{\phi}$ & The model that generate bids. \\ 
       $\overline{\alpha}_k$ & The cumulative product of $1-\beta_j, j=0,...,k$ \\
       $\beta_k$ & Schedualing factors. \\
       $\alpha_k$ & 1-$\beta_k$. \\
    \bottomrule
    \end{tabular}
    \label{tab:notation}
\end{table}

\subsection{Diffusion Modeling}
\label{sec:diffusion_modeling}
As a kind of generative model, diffusion models~\cite{ho2020denoising,vincent2011connection} use the diffusion process to gradually denoise latent samples to generate the new sample and have been widely used in generating pictures, videos, and audio. One of the widely used diffusion models, denoising diffusion probabilistic model~(DDPM), consists of two processes:

~\

\noindent\textbf{Forward process.} In the forward process, the noise is gradually added to the latent variable, which is parameterized by a Markov chain with the transition $q(\boldsymbol{x}_k|\bs{x}_{k-1})=\mathcal{N}\left (x_k;\sqrt{1-\beta_k}\bs{x}_{k},\beta_kI \right )$, where $k\in\{1,...K\}$ refers to the diffusion step, and $\beta_k\in(0,1)$ is a pre-defined scale that controls the noise scale at step $k$. By defining $\overline{\alpha}_k=\prod_{i=1}^k\alpha_i=\prod_{i=1}^k(1-\beta_i)$, we can have the conditional distribution:
\begin{equation}
    q(x_k|x_{0})=\mathcal{N}\left(x_k;\sqrt{\overline{\alpha}_k}x_0,(1-\overline{\alpha}_k)I\right)
\end{equation}
In this paper, we apply cosine noise schedule to control the noise by:
\begin{equation}
    \overline{\alpha}_k=\frac{g(t)}{g(0)}=\frac{\text{cos}\left ( \frac{k/K+\gamma}{1+\gamma}\cdot\frac{\pi}{2} \right )}{\text{cos}\left ( \frac{\gamma}{1+\gamma}\cdot\frac{\pi}{2} \right )},
\end{equation}
where $\gamma$ is a constant.
When $K\rightarrow \infty$, $q(x_K)$ approaches to a standard Gaussian distribution~\cite{ho2020denoising}. Given the original trajectory $x_0$ and $\epsilon\sim\mathcal{N}(0,I)$, we have the noisy version at $k$ by $ x_k=\sqrt{\overline{\alpha}_k}x_0+\epsilon\sqrt{1-\overline{\alpha}_k} $

~\

\noindent\textbf{Reverse process.} In the reverse process, diffusion models plan to remove the added noise on $x_k$ and recursively recover $x_{k-1}$. To achieve this goal, a Gaussian distribution parameterized by $p_\theta(\boldsymbol{x}_{k-1}|\boldsymbol{x}_k)=\mathcal{N}\left( \boldsymbol{x}_{k-1}|\boldsymbol{\mu}_\theta(\boldsymbol{x}_k,k),\boldsymbol{\Sigma}_\theta(\boldsymbol{x}_k,k) \right)$ is learned, where $\boldsymbol{\mu}_\theta(\boldsymbol{x}_k,k)$ is the learned mean and $\boldsymbol{\Sigma}_\theta(\boldsymbol{x}_k,k)$ is the learned covariance of the Gaussian distribution parameterized by a neural network with parameter $\theta$. For generating new samples, we can simply use re-parameterization trick~\cite{kingma2015variational} to sample a noise $\boldsymbol{x}_K\sim \boldsymbol{\mu}_K+\epsilon \sigma_K$ and recursively denoise the sample by $p_\theta(\bs{x}_{k-1}|\bs{x}_k)$ for generation.

~\

\noindent\textbf{Optimization.} DDPM optimizes the Evidence Lower BOund~(ELBO) of generative models. In the context of the diffusion model, we can take the latent samples as hidden variables and rewrite ELBO in the following form:
\begin{equation}
\small
    \begin{split}
        & \mathbb{E}_q\left [-\log p_\theta(\boldsymbol{x_0})\right ] \\
        \le &\mathbb{E}_q\left [ -\text{log}\frac{p_\theta(\boldsymbol{x}_{0:K})}{q(\boldsymbol{x}_{1:K}|\boldsymbol{x}_0)} \right ] \\
        =&\mathbb{E}_q\left [ D_{KL}( q(\boldsymbol{x}_K|\boldsymbol{x}_0)||p_\theta(\boldsymbol{x}_K)) \right ] -\mathbb{E}_q\left [ \text{log}p_\theta \left (\boldsymbol{x}_0|\boldsymbol{x}_1 \right ) \right ]  \\
        +& \mathbb{E}_q\left [ \sum_{t>1}D_{KL}\left (q(\boldsymbol{x}_{k-1}|\boldsymbol{x}_k,\boldsymbol{x}_0)||p_\theta(\boldsymbol{x}_{k-1}|\boldsymbol{x}_k) \right ) \right ], \\
    \end{split}
\end{equation}
\noindent where the first term has no learned variable given variance $\beta_k$ is fixed to constants, thus can be ignored during training. The second term is the reconstruction term where $p_\theta(\cdot)$ is trained to recover the original sample $\boldsymbol{x}_0$ from the noise sample $\boldsymbol{x}_{1}$. The last term is the denoising term where $p_\theta(\cdot)$ is trained to denoise $\boldsymbol{x}_k$ to get $\boldsymbol{x}_{k-1}$, thus we can recurrently denoise the latent samples. In the original paper~\cite{ho2020denoising} the author shows that the last term can be simplified to the noise prediction objective $\mathbb{E}_{k,\boldsymbol{x}_0,\boldsymbol{\epsilon}}\left [ ||\boldsymbol{\epsilon}-\boldsymbol{\epsilon}_\theta \left ( \sqrt{\overline{\alpha}_k}\boldsymbol{x}_0+\sqrt{1-\overline{\alpha}_k}\boldsymbol{\epsilon},k \right )|| \right ]$, where $\overline{\alpha}_k=\prod_{i=1}^{k}\alpha_k=\prod_{i=1}^{k}(1-\beta_k)$.

\subsection{Model Configuration}

We parameterize the noise model $\epsilon_\theta$ with a temporal U-Net~\cite{ronneberger2015u}, consisting of 3 repeated residual blocks. Each block is consisted of two temporal convolutions, followed by group normalization~\cite{wu2018group}, and a final Mish activation function~\cite{misra2019mish}. Timestamp and condition embeddings, both 128-dimensional vectors, are produced by separate 2-layered MLP~(with 256 hidden units and Mish activation function) and are concatenated together before getting added to the activation functions of the first temporal convolution within each block. $f_\phi$ is parameterized with a 3-layer MLP. 

\subsection{Pseudo-code}

The process of training and inference of DiffBid is shown in Algorithm~\ref{algo:training} and Algorithm~\ref{algo:generation} respectively. 

\begin{algorithm}[h]
\caption{Bid Generation with DiffBid.}
\begin{algorithmic}[1]
    \Require noise model $\epsilon_\theta$, inverse dynamics $f_\phi$, guidance scale $\omega$, condition $\bs{y}$, max diffusion step $K$, scales $\beta_t, \ t=1,...,K$.
    \Ensure Bidding parameters $\bs{a}_t$
    \State Get history of states $\bs{s}_{0:t}$;
        \State Sample $\bs{x}_K(\tau)\sim\mathcal{N}
        (0,\beta_K I)$;
        \For{$k=K,...,0$}
            \State $\bs{x}_k(\tau)[:t]\leftarrow\bs{s}_{0:t}$
            \State Estimating noise $\hat{\epsilon}$ through Eq.~(\ref{eq:perturbed_noise})
            \State $\left(\mu_{k-1},\Sigma_{k-1}\right)\leftarrow\text{Denoise}(\bs{x}_k(\tau),\hat{\epsilon}_k)$
            \State $\bs{x}_{k-1}\sim\mathcal{N}(\mu_{k-1},\alpha\Sigma_{k-1})$
        \EndFor
        \State Extract $(\bs{s}_{t-L:t},\bs{s}'_{t+1})$ from $\bs{x}_0(\tau)$
        \State Generate $\bs{\hat{a}}_t=f_\phi(\bs{s}_{t-L:t},\bs{s}'_{t+1})$; \\
    \Return $\bs{\hat{a}}_t$.
\end{algorithmic}
\label{algo:generation}
\end{algorithm}

\begin{algorithm}[h]
\caption{Training of DiffBid.}
\begin{algorithmic}[1]
    \Require randomly initialized $\theta$, $\phi$, bidding trajectory set $\mathcal{D}$
    \Ensure optimized $\theta$, $\phi$ 
    \While{not converge}
        \State Sample a batch of trajectories $\mathcal{B}\in\mathcal{D}$;
        \For{\textbf{all} $\tau\in\mathcal{B}$}
        \State Sample $k\sim\text{Uniform}(1,K)$, $\epsilon\sim\mathcal{N}(0,I)$;
        \State Compute $\bs{x}_k(\tau)$ via $q(\bs{x}_k(\tau)|\bs{x}_0(\tau))$ in Eq~(\ref{eq:add_noise});
        \State Compute $\mathcal{L}(\theta,\phi)$ by Eq~(\ref{eq:training});
        \State Perform gradient descent to optimize $\theta$ and $\phi$;
        \EndFor
    \EndWhile \\
    \Return optimized $\theta$, $\phi$
\end{algorithmic}
\label{algo:training}
\end{algorithm}

\subsection{Analytical Results for Action Control}
\label{sec:action_control}
\begin{figure}[t]
    \centering
        \includegraphics[width=.8\linewidth]{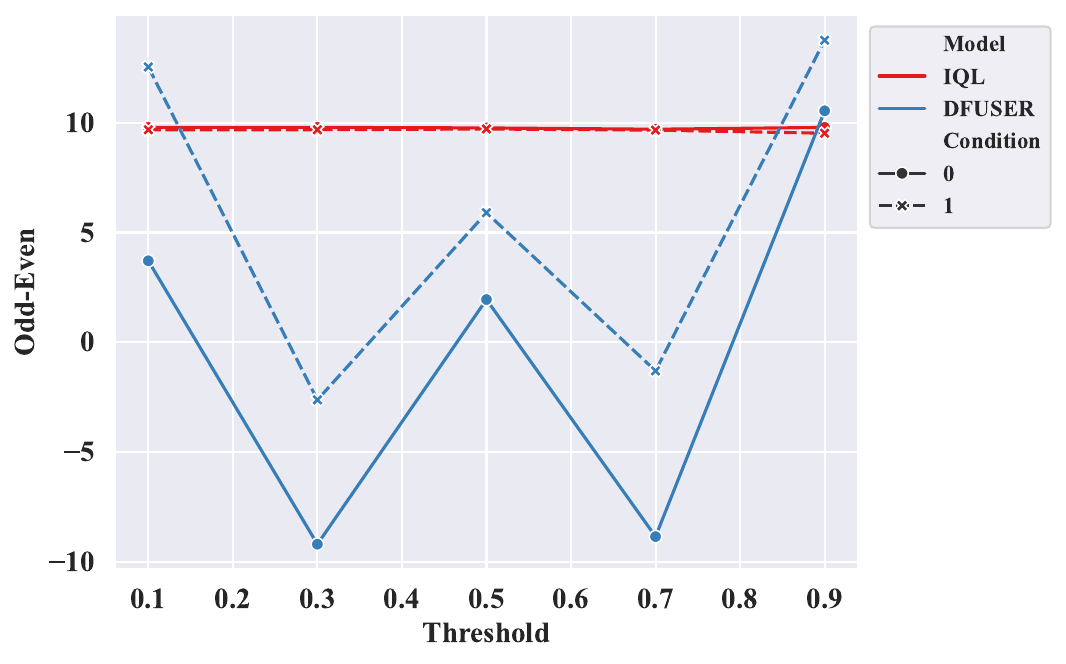}
        \label{fig:control_of_action}
    \caption{Ability of Action Control.}
    \label{fig:control_action}
\end{figure}
We analyze the ability of different models in controlling actions. To achieve this goal, we re-define the return function to be the summation of actions in odd time steps minus the summation of actions in even time steps. The results are shown in Figure~\ref{fig:control_action}. We find DiffBid can better control the action than IQL. The main reason is that controlling of actions is difficult for RL in long horizons. Instead, DiffBid directly models the correlation of trajectories and returns, thus can well handle the long trajectory situation.

\subsection{Statistical Analyses for Bidding Trajectory}

The study by \cite{DBLP:conf/icml/HaoPMWJHCBXXZYL20} indicates that CE follows a power-law decline as the number of winning impressions increases. Our statistical analysis confirms that this finding holds true at every discrete time step, with decay rates varying temporally due to the heterogeneous nature of the impressions. Figure~\ref{fig:cost_effe} shows three steps sampled from the online advertising system, From which another key insight is that the optimal bidding strategy's $ce*$ is equivalent to selecting a specific number of winning impressions per time step. We denote the number at time step $t$ as $n_t$. 

Another finding illustrated in Figure~\ref{fig:fluctua} is that the costs of impressions remain relatively stable throughout the total episode, fluctuating by less than 5\%. This stability allows us to approximate the cost of each impression $c_i$ with the average cost $\bar{c} = \frac{1}{N}\sum_{i=0}^{N} c_i$, where $N$ represents the number of winning impressions of the total episode. Therefore, the total cost at each time step $c_t = n_t \cdot \bar{c}$. In auto-bidding modeling, $c_t$ can be calculated from the state trajectory by using the difference in the remaining budget between two consecutive steps. Consequently, we can conclude that the optimal strategy correlates to a specific state trajectory.

\subsection{Theoretical Analysis}
\label{ap:theory}
We first give the definition of several decision process and then show the theoretical analysis.
\begin{definition}[Markovian Decision Process~(MDP)]
MDP is a stochastic mapping from a state-action pair to state-reward pairs. Formally, $\mathcal{T}: \mathcal{S}\times\mathcal{A}\rightarrow \mathcal{S}\times\mathcal{R}$, where $\mathcal{T}$ denotes a stochastic mapping.
\end{definition}
\begin{definition}[History-based Decision Process~(HDP)]
HDP is a stochastic mapping from a history-action pair to observation-reward pairs. Formally, $\mathcal{P}: \mathcal{H}^* \times \mathcal{A}\rightarrow \mathcal{O} \times \mathcal{R}$, where $\mathcal{P}$ denotes a stochastic mapping.
\end{definition}

We show that a sequential decision-making problem can be constructed to maximize the same objective. The main results are given by~\cite{qin2023learning} and we put the proofs here for completeness. To start, let the ground-truth distribution of demonstrations be $p^*(\bs{x}_0(\tau))$ and the learned marginal distributions of state sequences be $p_\theta(\bs{x}_0(\tau))$. Then Eq.~(\ref{eq:objective}) is an empirical estimation of
\begin{equation}
\begin{split}
    \mathbb{E}_{p^*(\bs{s}_0)}\left [ \text{log}p^*(\bs{s}_0)+\mathbb{E}_{p^*(\bs{s}_{1:T}|\bs{s}_0)}\left [ \text{log}p_\theta(\bs{s}_{1:T}|\bs{s}_0) \right] \right ]
\end{split}
\end{equation}
 Suppose the MLE yields the maximum, we will have $p_\theta^*=p^*$. Then we define $V^*(s_0):=\mathbb{E}_{p^*(s_{1:T}|s_0)}[\text{log}p^*(s_{1:T}|s_0)]$, and generalize it to have a $V$ function:
\begin{equation}
    V^*(s_{0:t})=\mathbb{E}_{p^*(s_{t+1:T}|s_{0:t})}[\text{log}p^*(s_{t+1:T}|s_{0:t})]
\end{equation}
which comes with a Bellman optimal equation:
\begin{equation}
    V^{*}(s_{0:t}):=\mathbb{E}_{p^*(s_{t+1}|s_{0:t})}[r(s_{t+1}, s_{0:t})+V^*(s_{0:t+1})]
\end{equation}
with $r(s_{t+1}, s_{0:t}) := \log p^*(s_{t+1}|s_{0:t}) = \log p^*_{a}(s_{t}|s_{0:t})p^*(s_{t+1}|s_t, a_t)dt$, $V^*(s_{0:T}) := 0$. It is worth noting that the $r$ defined above involves the optimal policy, which may not be known a priori. We can resolve this by replacing it with $r_{\alpha}$ for an arbitrary policy $p_{\alpha}(a_t|s_{0:t})$. All Bellman identities and updates should still hold. The entailed Bellman update, \textit{value iteration}, for arbitrary $V$ and $\alpha$ is
\begin{equation}
    V(s_{0:t}) = \mathbb{E}_{p^*(s_{t+1}|s_{0:t})}[ r_{\alpha}(s_{0:t}, s_{t+1}) + V(s_{0:t+1})].
\end{equation}
We then define $r(s_{t+1}, a_t, s_{0:t}) := r(s_{t+1}, s_{0:t}) + \log p^*_{a}(a_t|s_{0:t})$ to construct a $Q$ function:
\begin{equation}
    Q^*(a_t; s_{0:t}) := \mathbb{E}_{p^*(s_{t+1}|s_{0:t})}[ r(s_{t+1}, a_t, s_{0:t}) + V^*(s_{0:t+1})],
\end{equation}
which entails a Bellman update, \textit{Q backup}, for arbitrary $\alpha$, $Q$ and $V$
\begin{equation}
    Q(a_t; s_{0:t}) = \mathbb{E}_{p^*(s_{t+1}|s_{0:t})}[ r_{\alpha}(s_{0:t}, a_t, s_{t+1}) + V(s_{0:t+1})].
\end{equation}
Also note that the $V$ and $Q$ in identities Eq. (23) and Eq. (25) respectively are not necessarily associated with the policy $p_{\alpha}(a_t|s_{0:t})$. Slightly overloading the notations, we use $Q_{\alpha}, V_{\alpha}$ to denote the expected returns from policy $p_{\alpha}(a_t|s_{0:t})$. By now, we finish the construction of atomic algebraic components and move on to check if the relations between them align with the algebraic structure of a sequential decision-making problem. We first prove the construction above is valid at optimality.
\begin{lemma}
    When \( f_{\alpha}(a_{t}; s_{0:t}) = Q^*(a_{t}; s_{0:t}) - V^*(s_{0:t}), p_{\alpha}(a_{t}|s_{0:t}) \) is the optimal policy.
\end{lemma}
\begin{proof}
Note that the construction gives us
\begin{equation}
    \begin{split}
    & Q^*(a_{t}; s_{0:t}) \\ 
    =&\mathbb{E}_{p^*(s_{t+1}|s_{0:t})} \left[ r(s_{t+1}, s_{0:t}) + \log p^*_{\alpha}(a_{t}|s_{0:t}) + V^*(s_{0:t+1}) \right] \\
    =&\log p^*_{\alpha}(a_{t}|s_{0:t}) + \mathbb{E}_{p^*(s_{t+1}|s_{0:t})} \left[ r(s_{t+1}, s_{0:t}) + V^*(s_{0:t+1}) \right] \\
    =&\log p^*_{\alpha}(a_{t}|s_{0:t}) + V^*(s_{0:t})
\end{split}
\end{equation}
\end{proof}
Obviously, \( Q^*(a_{t}; s_{0:t}) \) lies in the hypothesis space of \( f_{\alpha}(a_{t}; s_{0:t}) \). It indicates that we need to either parameterize \( f_{\alpha}(a_{t}; s_{0:t}) \) or \( Q(a_{t}; s_{0:t}) \). While \( Q^{\alpha} \) and \( V^{\alpha} \) are constructed from the optimality, the derived \( Q^{\alpha} \) and \( V^{\alpha} \) measure the performance of an interactive agent when it executes with the policy \( p_{\alpha}(a_{t}|s_{0:t}) \). They should be consistent.
\begin{lemma}
    \( V^{\alpha}(s_{0:t}) \) and \( \mathbb{E}_{p_{\alpha}(a_{t}|s_{0:t})} \left[ Q^{\alpha}(a_{t}; s_{0:t}) \right] \) yield the same optimal policy \( p^*_{\alpha}(a_{t}|s_{0:t}) \)
\end{lemma}
\begin{proof}
\begin{equation*}
\small
\begin{split}
    & \mathbb{E}_{p_{\alpha}(a_{t}|s_{0:t})} \left[ Q^{\alpha}(a_{t}; s_{0:t}) \right]  \\
    :=& \mathbb{E}_{p_{\alpha}(a_{t}|s_{0:t})} \left[ \mathbb{E}_{p^*(s_{t+1}|s_{0:t})} \left[ r(s_{t+1}, a_{t}, s_{0:t}) + V^{\alpha}(s_{0:t+1}) \right] \right] \\
    =&\mathbb{E}_{p_{\alpha}(a_{t}|s_{0:t})} \mathbb{E}_{p^*(s_{t+1}|s_{0:t})} \left[ \log p^*_{\alpha}(a_{t}|s_{0:t}) + r(s_{t+1}, s_{0:t}) + V^{\alpha}(s_{0:t+1}) \right] \\
    =&\mathbb{E}_{p^*(s_{t+1}|s_{0:t})} \left[ r(s_{t+1}, s_{0:t}) - H_{\alpha}(a_{t}|s_{0:t}) + V^{\alpha}(s_{0:t+1}) \right] \\
     &-\sum_{k=t+1}^{T-1} \mathbb{E}_{p^*(s_{t+1:k}|s_{0:t})} \left[ H_{\alpha}(a_{k}|s_{0:k}) \right]
\end{split}
\end{equation*}
where $\mathcal{H}(\cdot)$ is the entropy term. The last line is derived by recursively applying the Bellman equation in the line above until \( s_{0:T} \). As an energy-based policy, \( p_{\alpha}(a_{t}|s_{0:t}) \)'s entropy is inherently maximized~\cite{jaynes1957information}. Therefore, within the hypothesis space, \( p_{\alpha}^*(a_{t}|s_{0:t}) \) that optimizes \( V^{\alpha}(s_{0:t}) \) also leads to the optimal expected return \( \mathbb{E}_{p_{\alpha}(a_{t}|s_{0:t})} \left[ Q^{\alpha}(a_{t}; s_{0:t}) \right] \). 
\end{proof}

Given the convergence proof by Ziebart~\cite{10.5555/2049078}, we have:
\begin{lemma}
    If \( p^*(s_{t+1}|s_{0:t}) \) is accessible and \( p^*_{\gamma}(s_{t+1}|s_t, a_t) \) is known, soft policy iteration and soft \( Q \) learning both converge to \( p^*_{\alpha}(a_t|s_{0:t}) = p^*_{\alpha}(a_t|s_{0:t}) \propto \exp(Q^*(a_t; s_{0:t})) \) under conditions.
\end{lemma}
 Lemma 3 means given \( p^*(s_{t+1}|s_{0:t}) \) and \( p^*_{\gamma}(s_{t+1}|s_t, a_t) \), we can recover \( p^*_{\alpha} \) through reinforcement learning methods, instead of the proposed MLE. So \( p_{\alpha}(a_t|s_{0:t}) \) is a viable policy space for the constructed sequential decision-making problem. Together, Lemma A.1, Lemma A.2 and Lemma A.3 provide proof for a valid sequential decision-making problem that maximizes the same objective of MLE, by Lemma A.4.

\begin{lemma}[MLE as non-Markovian decision-making process]
Assuming the Markovian transition $p_{\gamma^*}(s_{t+1}|s_t, a_t)$ is known, the ground-truth conditional state distribution $p^*(s_{t+1} |s_{0:t})$ for demonstration sequences is accessible, we can construct a non-Markovian sequential decision-making problem, based on a reward function $r_{\alpha}(s_{t+1},s_{0:t}):={\rm{log}}\int p_{\alpha}(a_t|s_{0:t})p_{\gamma^*}(s_{t+1}|s_t,a_t)d a_t$ for an arbitrary energy-based policy $p_\alpha(a_t|s_{0:t})$. Its objective is
\begin{equation}
    \sum_{t=0}^T\mathbb{E}_{p^*(s_{0:t})}\left [ V^{p_\alpha}(s_{0:t}) \right ]=\mathbb{E}_{p^*(s_{0:T})}\left [ \sum_{t=0}^{T}\sum_{k=t}^T r_{\alpha}(s_{k+1};s_{0:k}) \right ]
\end{equation}
$V^{p_\alpha}(s_{0:t}):=\mathbb{E}_{p^*(s_{t+1:T}|s_{0:t})}[\sum_{k=1}^Tr_\alpha(s_{k+1};s_{0:k})]$ is the value function of $p_\alpha$. This objective yields the save optimal policy as the Maximum Likelihood Estimation  $\mathbb{E}_{p^*(s_{0:T})}\left [{\rm{log}}
 p_\theta(s_{0:T})\right ]$.
 \end{lemma}

\begin{table}[t]
    \centering
     \caption{Parameters of the Real Advertising System.}
    \resizebox{.9\linewidth}{!}{
    \begin{tabular}{l|c}
    \toprule
       Parameters  & Values \\
    \midrule
       Number of advertisers  & 30 \\
       Time steps in an episode, $T$ & 96 \\
       Minimum number of impression opportunities $N_{\text{min}}$ & 50 \\
       Maximum number of impression opportunities $N_{\text{max}}$ & 300 \\
       Minimum budget & 1000 Yuan \\
       Maximum budget & 4000 Yuan \\
       Value of impression opportunities in stage 1, $v_{j,t}^1$ & 0$\sim$ 1 \\
       Value of impression opportunities in stage 2, $v_{j,t}^2$ & 0$\sim$ 1 \\
       Minimum bid price, $A_{\text{min}}$ & 0 Yuan \\ 
       Maximum bid price, $A_{\text{max}}$ & 1000 Yuan \\
       Maximum value of impression opportunity, $v_M$ & 1 \\
       Maximum market price, $p_M$ & 1000 Yuan \\
    \bottomrule
    \end{tabular}}
    \label{tab:environment}
\end{table}

\begin{figure}[t]
    \centering
    \subfigure[Cost Effectiveness Curve Samples from Three Steps]{
        \includegraphics[width=.47\linewidth]{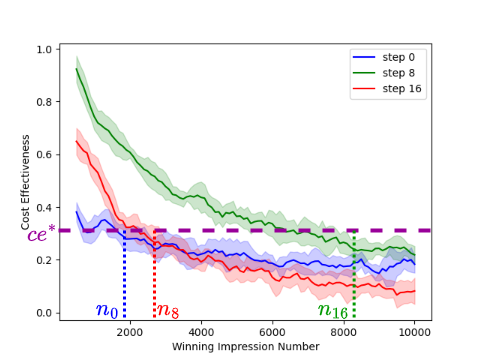}
        \label{fig:cost_effe}
    }
    \subfigure[Impression Cost Fluctuation at Different Time Steps]{
        \includegraphics[width=.47\linewidth]{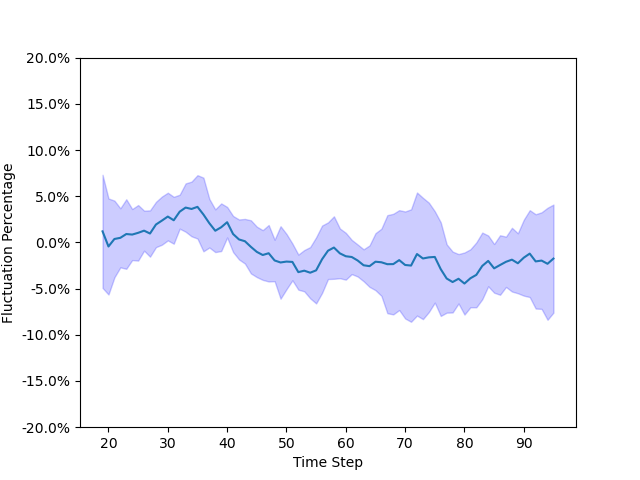}
        \label{fig:fluctua}
    }
    \caption{Statistical Results From Online Advertising System.}
    \label{fig:statisti}
\end{figure}

\end{document}